\newif\ifisarxiv
\isarxivtrue

\documentclass{article}

\usepackage{microtype}
\usepackage{graphicx}
\usepackage{subfigure}
\usepackage{booktabs} 
\usepackage{adjustbox}
\usepackage{multirow}
\usepackage{enumitem}
\usepackage{tikz}
\usepackage{xfrac}
\usetikzlibrary{shapes.geometric, arrows}
\usetikzlibrary{decorations.pathreplacing, chains}
\usetikzlibrary{fit}
\usetikzlibrary{calc}
\usetikzlibrary{backgrounds}
\usetikzlibrary{patterns}
\usetikzlibrary{positioning,arrows.meta,bending}

\usepackage{hyperref}
\usepackage{listings}

\ifisarxiv
\usepackage{fullpage}
\newcommand{\citet}[1]{\cite{#1}}
\renewcommand{\vspace}[1]{}
\newlength{\mywidth}
\else
\usepackage[accepted]{icml2021}
\newlength{\mywidth}
\fi

\usepackage{color}
\usepackage{bm}
\usepackage{hyperref}
\usepackage{amsmath,amsfonts,amsthm}
\usepackage{blindtext}
\usepackage{listings}
\usepackage{array}

\usepackage{ifthen}

\newif\ifdebug
\debugfalse

\ifdebug
\usepackage[textsize=tiny,color=lightgray,textwidth=2cm]{todonotes}
\paperwidth=\dimexpr \paperwidth + 4cm\relax
\oddsidemargin=\dimexpr\oddsidemargin + 2cm\relax
\evensidemargin=\dimexpr\evensidemargin + 2cm\relax
\else
\usepackage[disable]{todonotes}
\fi

\newtheorem{lemma}{Lemma}
\newtheorem{theorem}{Theorem}

\newtheorem{proposition}{Proposition}

\newcommand{\ceil}[1]{\left\lceil #1 \right\rceil}
\newcommand{\floor}[1]{\left\lfloor #1 \right\rfloor}

\newcommand{\Hl}[1]{{\Hv^{(#1)}}}

\newcommand{\hHlT}[1]{{{{}\hat{ \Hv}^{(#1)}}^\top}}
\newcommand{\bnl}{b_n^{(l)}}
\newcommand{\Bnl}{B_n^{(l)}}
\newcommand{\Ld}{\Lc_{\Dc}}

\newcommand{\Cl}[1]{{\Cv^{(#1)}}}
\newcommand{\hCl}[1]{{\hat\Cv^{(#1)}}}

\newcommand{\Thetal}[1]{{\Thetav^{(#1)}}}

\newcommand{\F}[2]{{\Fv^{(#1)}}\left(#2\right)}
\newcommand{\G}[2]{{\Gv^{(#1)}}\left(#2\right)}
\newcommand{\GH}[2]{{\Gv_{\Hv}^{(#1)}}\left(#2\right)}
\newcommand{\GT}[2]{{\Gv_{\Thetav}^{(#1)}}\left(#2\right)}

\newcommand{\hn}{\hat\nabla}

\newcommand{\vect}[1]{\boldsymbol{\mathbf{#1}}}

\newcommand{\argmax}{\operatornamewithlimits{argmax}}

\newcommand{\E}[1]{\mathbb{E}\left[#1\right]}
\newcommand{\Econd}[2]{\mathbb{E}\left[#1\;\middle|\;#2\right]}
\newcommand{\Var}[1]{\mathrm{Var}\left[#1\right]}

\newcommand{\Varcond}[2]{\mathrm{Var}\left[#1\;\middle|\;#2\right]}

\newcommand{\Thetav}{\vect\Theta}

\newcommand{\bv}{\vect b}

\newcommand{\hv}{\vect h}

\newcommand{\mv}{\vect m}

\newcommand{\sv}{\vect s}

\newcommand{\uv}{\vect u}

\newcommand{\wv}{\vect w}
\newcommand{\xv}{\vect x}
\newcommand{\yv}{\vect y}

\newcommand{\Cv}{\vect C}

\newcommand{\Fv}{\vect F}
\newcommand{\Gv}{\vect G}
\newcommand{\Hv}{\vect H}

\newcommand{\Rv}{\vect R}

\newcommand{\Wv}{\vect W}
\newcommand{\Xv}{\vect X}
\newcommand{\Yv}{\vect Y}

\newcommand{\Bc}{\mathcal B}

\newcommand{\Dc}{\mathcal D}

\newcommand{\Lc}{\mathcal L}

\newcommand{\Uc}{\mathcal U}

\newcommand{\Eb}{\mathbb E}

\newcommand{\Ib}{\mathbb I}

\newcommand{\Rb}{\mathbb R}

\newcommand{\norm}[1]{\left\lVert#1\right\rVert}

\makeatletter
\newtheorem*{rep@theorem}{\rep@title}
\newcommand{\newreptheorem}[2]{%
	\newenvironment{rep#1}[1]{%
		\def\rep@title{#2 \ref{##1}}%
		\begin{rep@theorem}}%
		{\end{rep@theorem}}}
\makeatother

\newcommand{\method}{ActNN\xspace}

\ifisarxiv
\else
\icmltitlerunning{\method: Reducing Training Memory Footprint via 2-Bit Activation Compressed Training}
\fi

\ifisarxiv
\title{\method: Reducing Training Memory Footprint\\ via 2-Bit Activation Compressed Training}
\author{%
  Jianfei Chen\footnote{Equal contribution.}, Lianmin Zheng$^*$, Zhewei Yao, Dequan Wang \\
  Ion Stoica, Michael W. Mahoney, and Joseph E. Gonzalez \\
  University of California, Berkeley\\
  \texttt{\{jianfeic, lmzheng\}@berkeley.edu} \\
}
\date{}
\fi

\begin{document}

\ifisarxiv
\maketitle
\fi

\begin{abstract}
The increasing size of neural network models has been critical for improvements in their accuracy, but device memory is not growing at the same rate. 
This creates fundamental challenges for training neural networks within limited memory environments. 
In this work, we propose \method, a memory-efficient training framework that stores randomly quantized activations for back propagation. 
We prove the convergence of \method for general network architectures, and we characterize the impact of quantization on the convergence via an exact expression for the gradient variance. 
Using our theory, we propose novel mixed-precision quantization strategies that exploit the activation's heterogeneity across feature dimensions, samples, and layers. 
These techniques can be readily applied to existing dynamic graph frameworks, such as PyTorch, simply by substituting the layers. 
We evaluate \method on mainstream computer vision models for classification, detection, and segmentation tasks. 
On all these tasks, \method compresses the activation to 2 bits on average, with negligible accuracy loss. 
\method reduces the memory footprint of the activation by 12$\times$, and it enables training with a $6.6 \times$ to $14 \times$ larger batch size.
We implement ActNN as a PyTorch library at \url{https://github.com/ucbrise/actnn}.
\end{abstract}

\section{Introduction}

\begin{figure}[t]
\centering
\includegraphics[width=0.9\mywidth]{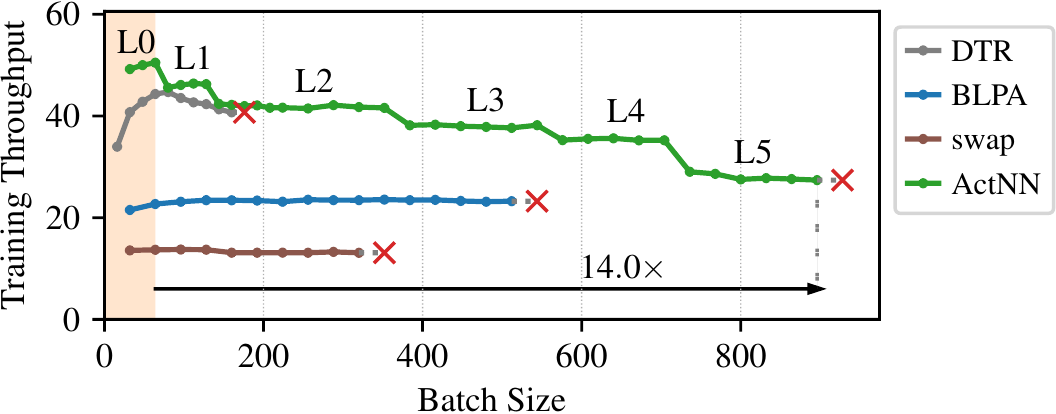}
\vspace{-.6em}
\caption{\small Batch size vs. training throughput on ResNet-152. Red cross mark means out-of-memory. The shaded yellow region denotes the possible batch sizes with full precision training.
\method achieves significantly larger maximum batch size over other state-of-the-art systems and displays a nontrivial trade-off curve.}
\label{fig:intro}
\vspace{-1.5em}
\end{figure}

Within the last three years, state-of-the-art machine learning models have become over 4,000 times larger~\cite{devlin2018bert,fedus2021switch}. 
On the other hand, the memory capacity  of GPUs has increased relatively slowly, remaining on the order of tens of gigabytes.
This creates a fundamental barrier to the development and training of neural networks.



Activation compressed training (ACT) is a promising approach to reduce the training memory footprint~\cite{chakrabarti2019backprop,fu2020don}.
During training, all layers' activations need to be kept in the memory for computing the gradients. 
ACT saves memory by compressing activations to lower numerical precision via quantization. 
It was proposed in BLPA~\cite{chakrabarti2019backprop}, and it was later extended by TinyScript~\cite{fu2020don} with non-uniform quantization strategies. 
These prior works succeeded in training ResNet-50 with 4-bit activations. 

However, applications of ACT are hindered by several drawbacks. 
First, the convergence behavior of ACT methods is not well understood (except for an analysis for multi-layer perceptrons~\cite{fu2020don}, under strong mean-field assumptions~\cite{yang2017mean}). 
Second, prior works mostly focus on dedicated architectures, e.g., a customized version of the pre-activation ResNet~\cite{he2016identity}, limiting their generality.
Third, existing quantization strategies are not specifically designed for ACT, making their compression ratio~suboptimal. 
%

In this work, we propose \method, a framework for ACT that overcomes all the challenges. 
\method stores randomly quantized activations to compute the gradients. 
Theoretically, we view the gradient computed by \method (``\method gradient'') as a stochastic approximation of the gradient computed with the full-precision activation (``FP gradient''). 
We show that the \method gradient is an unbiased estimator of the FP gradient, and  we prove \method 's convergence \emph{for general model architectures}. 
This enables one to apply ACT to general problems with theoretical guarantees.

We characterize the impact of quantization on the convergence via an exact expression for the gradient variance. 
Better quantization strategies reduce the gradient variance, and can achieve satisfactory convergence with fewer bits. 
Inspired by the theory, we design novel quantization strategies to exploit activations' heterogeneity across feature dimensions, samples, and layers. 
This includes a per-group quantizer and a fine-grained mixed precision algorithm, which approximately minimizes the gradient variance under a given memory budget. 
\method tunes the quantization strategy on-the-fly. On a wide range of tasks, including image classification, semantic segmentation, and object detection, \method compresses activations to 2 bits, with negligible ($<0.5\%$) accuracy loss. \method even converges and produces reasonable results with only 1.25-bit activations. 
This improves significantly from prior work~\cite{chakrabarti2019backprop,fu2020don}, which only converges with 4 bits. 

We implement our method as a library based on PyTorch. The library consists of a collection of activation compressed layers.
Memory saving training can be achieved with simply layer substitution, e.g., replace \texttt{torch.nn.Conv2d} with \texttt{actnn.Conv2d}.
The library also provides several optimization levels to exploit the trade-off between memory saving and training speed. In practice, \method reduces the activation memory by 12$\times$, enabling training with a $6.6\times$ to $14\times$ larger batch size on the same GPU. We compare \method with existing systems, where \method achieves a much larger batch size (Fig.~\ref{fig:intro}). 
\method also enables training larger models without additional computational resources. With a fixed amount of memory, \method scales the training of ResNet to $6.4\times$ deeper, or $3.7 \times$ wider, or $3.1 \times$ higher resolution. 


To summarize, our contributions are in three folds:\vspace{-.5em}
\begin{enumerate}[nosep]
	\item A general convergence theory for ACT;
	\item An heterogeneity-aware quantization strategy that achieves 2-bit compression;
	\item An efficient implementation of activation compressed layers in PyTorch.
\end{enumerate}



\tikzstyle{layer} = [rectangle, text centered, draw=black, fill=red!30, style={inner sep=0,outer sep=0}, minimum height=0.35cm]
\tikzstyle{augmented}=[rectangle, draw=black, pattern=north east lines, pattern color=black]
\tikzstyle{arrow} = [thick,->,>=stealth]
\tikzstyle{layertext} = [text centered]
\tikzstyle{background}=[rectangle, draw=black!80!white, very thick, dashed, style={inner sep=0.2cm}]
\tikzstyle{backgroundlegend}=[rectangle, draw=black, dashed, style={inner sep=0.15cm}]
\tikzstyle{layerlegend} = [rectangle, draw=black, fill=red!30, style={inner sep=0.15cm}]
\tikzstyle{auglegend} = [rectangle, draw=black, pattern=north east lines, pattern color=black, style={inner sep=0.15cm}]

\tikzset{
	lbl/.style={font=\tiny},
	base/.style={on chain, on grid, align=center, minimum height=4ex, font=\small},
	forward/.style={base, draw, rectangle, thick, fill=green!50!white, text width=2em},
	lstyle/.style={base, draw, ellipse, thick, fill=purple!50!white, text width=.8em},
	backward/.style={base, draw, rectangle, thick, fill=cyan!50!white, text width=2em},
	txt/.style={base, text centered, text width=2em},
	legendtxt/.style={base, align=left, text width=8em},	
	smalltxt/.style={on chain, on grid, text centered, text width=2em, font=\tiny},
	compress/.style={draw, circle, minimum size=4mm, inner sep=0, fill=red!50!white},
	decompress/.style={draw,  diamond, minimum size=4mm, inner sep=0, fill=blue!50!white},
	context/.style={thick, dashed, red},
	mainpath/.style={thick,>=stealth},
	norm/.style={->, draw},
	free/.style={->, draw},
	cong/.style={->, draw},
	it/.style={font={\small\itshape}}
}

\begin{figure*}[t]
	\resizebox{\linewidth}{!}{
	\begin{tikzpicture}[
		    >=stealth,              
	start chain=going right,    
	node distance=8mm and 12mm, 
	every join/.style={norm},   
    ]    
    \node(h0)[on chain, on grid]{};
	\node(f1)[forward] {$\Fv^{(1)}$};
	\node(f2)[forward] {$\Fv^{(2)}$};		
	\node(fn)[txt] {$\cdots$};
	\node(fl)[forward] {$\Fv^{(L)}$};
	\node(loss)[lstyle] {$\Lc$};
	\node(gl)[backward, join] {$\Gv^{(L)}$};
	\node(gn)[txt, join] {$\cdots$};
	\node(g2)[backward, join] {$\Gv^{(2)}$};
	\node(g1)[backward, join] {$\Gv^{(1)}$};
	\node(t1)[smalltxt, above=of f1]{$\Thetal{1}$};
	\node(t2)[smalltxt]{$\Thetal{2}$};
	\node(tl)[smalltxt, above=of fl]{$\Thetal{L}$};
	\node(ntl)[smalltxt, above=of gl]{$\hn_{\Thetal{L}}$};
	\node(nt2)[smalltxt, above=of g2]{$\hn_{\Thetal{2}}$};
	\node(nt1)[smalltxt]{$\hn_{\Thetal{1}}$};	
	\draw [->, mainpath] (h0.east) -- node[above, lbl] {$\Hl{0}$} (f1);	
	\draw [->, mainpath] (f1.east) -- node[above, lbl] {$\Hl{1}$} (f2);	
	\draw [->, mainpath] (f2.east) -- node[above, lbl] {$\Hl{2}$} (fn);	
	\draw [->, mainpath] (fn.east) -- node[above, lbl] {$\Hl{L-1}$} (fl);
	\draw [->, mainpath] (fl.east) -- node[above, lbl] {$\Hl{L}$} (loss);
	\draw [->, mainpath] (loss.east) -- node[above, lbl] {$\nabla_{\Hl{L}}$} (gl);
	\draw [->, mainpath] (gl.east) -- node[above, lbl] {$\hn_{\Hl{L-1}}$} (gn);
	\draw [->, mainpath] (gn.east) -- node[above, lbl] {$\hn_{\Hl{2}}$} (g2);
	\draw [->, mainpath] (g2.east) -- node[above, lbl] {$\hn_{\Hl{1}}$} (g1);
	\draw [->, mainpath] (t1.south) -- (f1);
	\draw [->, mainpath] (t2.south) -- (f2);
	\draw [->, mainpath] (tl.south) -- (fl);
	\draw [->, mainpath] (gl.north) -- (ntl);
	\draw [->, mainpath] (g2.north) -- (nt2);
	\draw [->, mainpath] (g1.north) -- (nt1);
	\node(c1)[compress, below=2mm of f1] {};
	\node(d1)[decompress, below=2mm of g1] {};
	\draw [->] (f1.south) -- (c1);
	\draw [->, context] (c1.south) -- ++(0mm, -6mm) -| node [near end, lbl, xshift=-3mm]  {$\hCl{1}$} (d1.south);
	\draw [->] (d1.north) -- (g1.south);
	
	\node(c2)[compress, below=2mm of f2] {};
	\node(d2)[decompress, below=2mm of g2] {};
	\draw [->] (f2.south) -- (c2);
	\draw [->, context] (c2.south) -- ++(0mm, -2mm) -| node [near end, lbl, xshift=-3mm]  {$\hCl{2}$} (d2.south);
	\draw [->] (d2.north) -- (g2.south);

	\node(cl)[compress, below=2mm of fl] {};
	\node(dl)[decompress, below=2mm of gl] {};
	\draw [->] (fl.south) -- (cl);
	\draw [->, context] (cl.east) --  node [near end, lbl, yshift=2mm]  {$\hCl{L}$} (dl);
	\draw [->] (dl.north) -- (gl.south);	

	\node(memory)[below=12mm of fl]{};
	\begin{scope}[every edge/.append style={thick,->,>=stealth}]
		\begin{pgfonlayer}{background}
	\node [background, style={inner sep=0.2cm},
	fit=(fl) (cl) (tl) (memory),
	label=right:\scriptsize{}, rounded corners=5, fill=gray!5] {};
	\end{pgfonlayer}
	\end{scope}
	\draw [->,mainpath,very thick] (tl) ++(-3mm, 6mm) -- ++(6mm, 0mm);
	\end{tikzpicture}\hspace{1em}
	\begin{tikzpicture}[
>=stealth,              
start chain=going right,    
node distance=8mm and 12mm, 
every join/.style={norm},   
]    
\node(t1)[legendtxt]{Full-Precision Tensor};
\node(t1l)[left=6mm of t1]{};
\node(t2)[legendtxt, below=6mm of t1]{Compressed Tensor};
\node(t3)[legendtxt, below=6mm of t2]{Compressor};
\node(t4)[legendtxt, below=6mm of t3]{Decompressor};
\node(t5)[legendtxt, below=6mm of t4]{Active Memory};
\draw [<-,mainpath] (t1.west) ++(-2mm, 0mm) -- ++(-6mm, 0mm);
\draw [<-,context] (t2.west) ++(-2mm, 0mm) -- ++(-6mm, 0mm);
\node(c)[compress, left=2mm of t3] {};
\node(d)[decompress, left=2mm of t4] {};
\node(e)[left=3mm of t5]{};
\begin{scope}[every edge/.append style={thick,->,>=stealth}]
\begin{pgfonlayer}{background}
\node [background, style={inner sep=0.1cm},
fit=(e),
label=right:\scriptsize{}, rounded corners=5, fill=gray!5] {};
\end{pgfonlayer}
\end{scope}

\begin{scope}[every edge/.append style={thick,->,>=stealth}]
\begin{pgfonlayer}{background}
\node [rectangle, draw=black, thick, inner sep=0.2cm,
fit=(t1l) (t5),
label=right:\scriptsize{}] {};
\end{pgfonlayer}
\end{scope}
\end{tikzpicture}
}
\vspace{-.6cm}
	\caption{\footnotesize ActNN's computational graph. Nodes: operations; Edges: tensors.  Edges that intersect with the dashed box are kept in memory. \label{fig:architecture}}
\end{figure*}
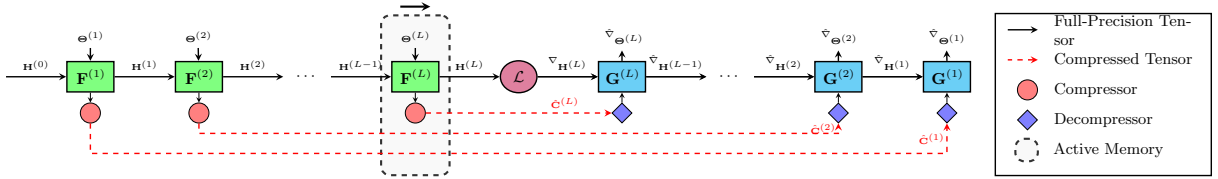

\section{Related Works}\label{sec:related}

\noindent \textbf{Quantized Training (QT)} Quantization-aware training~\cite{zhou2016dorefa,choi2018pact,Zhang_2018_ECCV,jacob2018quantization,dong2019hawqv2} or fully-quantized training~\cite{micikevicius2018mixed,wang2018training,chen2020statistical,sun2020ultra} aim to reduce the computational cost with quantization at the inference or training time. As a side effect, the training memory footprint can also be reduced. However, QT is a more challenging task to solve, as \emph{computational} kernels must directly support quantized tensors. In contrast, ACT only considers the \emph{storage}, and it can utilize more flexible quantization strategies for better compression. Furthermore, QT and ACT are complementary. One can utilize QT to accelerate the training, and apply ACT to further reduce the memory footprint. 

\noindent\textbf{Model / Gradient Compression} Model compression~\cite{han2015deep_compression} and gradient compression~\cite{lin2017deep} compress the weight and gradient to reduce the storage and communication overhead. However, activations have different properties with the weight and gradient, e.g., it has a ``sample'' axis. Moreover, ACT is more sensitive to the compression speed, as activations need to be compressed on-the-fly, and they are typically much larger than weights and gradients. \method 's compression strategy is designed specifically for these unique properties. 


\noindent\textbf{Memory-Efficient Training Systems}
Gradient checkpointing~\cite{chen2016training,jain2019checkmate,shah2020memory,kirisame2020dynamic} 
trades computation for memory by dropping some of the activations in the forward pass and recomputing them in the backward pass. Swapping~\cite{meng2017training,huang2020swapadvisor,wang2018superneurons,peng2020capuchin,ren2021zero} utilizes the huge amount of available CPU memory by swapping tensors between CPU and GPU. Model-parallel training~\cite{shoeybi2019megatron,lepikhin2020gshard,wang2019supporting} partitions the model across GPUs, so each GPU only stores a fraction of layers. All these methods save memory by storing fewer tensors in GPU. In contrast, \method compresses saved tensors, and is complementary to these approaches.

\section{Formulation and Theory}

In this section, we present a mathematical formulation of \method. 
Then, we establish its convergence by viewing it as a special case of stochastic gradient descent (SGD). 
The proofs of all theorems as well as a table of notations can be found in Appendix~\ref{sec:theorems}.

\subsection{Problem Formulation}
Consider training an $L$-layer neural network on a dataset $\Dc$. In each training iteration, we sample a minibatch $\small(\Xv, \Yv)$ from the dataset. 
Given the input $\small\Hl{0}=\Xv$, the $l$-th layer of the network is defined in a general form\vspace{-.5em}
\begin{align}
\small
\Hl{l} = \F{l}{\Hl{l-1}; \Thetal{l}},\label{eqn:fp}
\end{align}
\vspace{-1.5em}where $\small\Hl{l}$ is a $\small N\times D^{(l)}$-dimensional feature map, $N$ is the batch size, $\small D^{(l)}$ is the number of features, and $\small\Thetal{l}$ is a vector of parameters. 
Given the minibatch loss
$
\small\Lc = l(\Hl{L}, \Yv)
$, we compute the gradient $\small\nabla_{\Thetav^{(l)}}\Lc$, and update the parameter with SGD~\cite{bottou2010large}. 
Since the gradient is always taken with the loss $\Lc$, we simply denote the activation / parameter gradient as $\small\nabla_{\Hl{l}}$ and $\small\nabla_{\Thetav^{(l)}}$. 
To compute the gradient, 
the back-propagation can be expressed as\vspace{-.5em}
\begin{align}
\small
\nabla_{\Hl{l-1}}, \nabla_{\Thetal{l}}
 = \G{l}{\nabla_{\Hl{l}}, \Cv(\Hl{l-1}, \Thetal{l})},\label{eqn:fp-bp}
\end{align}
\vspace{-1.25em}
where $\Cv(\cdot)$ is the \emph{context}, i.e., the information that needs to be kept in  memory for back propagation. 
Essentially, the function $\small\G{l}{\cdot}$ takes the gradient of the output $\small\nabla_{\Hl{l}}$ and the context, and computes the gradient of the input. 
We refer this approach as full-precision (FP) training, and $\small\nabla_{\Hl{l}}$ and  $\small\nabla_{\Thetal{l}}$ as the FP gradient.
As a special case, consider a linear layer $\small\Hl{l}=\Hl{l-1}\Thetal{l}$ and its gradient
\vspace{-.5em}
\begin{align}
\small
\nabla_{\Hl{l-1}} = \nabla_{\Hl{l}} \Thetal{l}^\top,~~~
\nabla_{\Thetal{l}} = \Hl{l-1}^\top \nabla_{\Hl{l}}.\label{eqn:linear-bp}
\end{align}

\vspace{-1em}
In this case, we have $\small\Cv(\Hl{l-1}, \Thetal{l})=(\Hl{l-1}, \Thetal{l})$.

\subsection{Activation Compressed Training}

The context, in particular the activation, dominants the memory overhead for training neural networks on many tasks. 
To address this, instead of saving the full-precision context, \method saves a compressed version $\small\hat\Cv(\Hl{l-1}, \Thetal{l})$. 
In principle, any compression algorithm, either lossy or lossless, can be used here. 
In this paper, however, we solely focus on compressing by quantizing the context to lower numerical precision, since its overhead is relatively small. 
In this way, the compressed context $\small\hat\Cv$ is a lower precision version of the context $\small\Cv$. 
With the compressed context, we define the activation-compressed (AC) gradient as:
\begin{align}
\small
\hn_{\Hl{l-1}}, \hn_{\Thetal{l}}
 = \G{l}{\hn_{\Hl{l}}, \hat\Cv(\Hl{l-1}, \Thetal{l})},\label{eqn:ac-bp}
\end{align}
where $\small\hn_{\Hl{L}}=\nabla_{\Hl{L}}$. 
\method uses the AC gradient to update parameters. 
See Fig.~\ref{fig:architecture} for an illustration. 
Notice that, FP training and \method share the same forward propagation Eq.~(\ref{eqn:fp}), so their behavior is identical at inference time.

%

\subsection{Convergence Theory}

Now we study the convergence of \method. 
Assume that $\small\hat\Cv$ is quantized randomly, such that $\small\hat\Cv$ can be viewed as a stochastic estimator of $\small\Cv$. 
In this way, both FP and \method can be considered as SGD algorithms, with different stochastic gradients. 
Formally, let $\small\Thetav_t=\{\Thetal{l}\}_{l=1}^L$ be a flattened vector of parameters at the $t$-th iteration, and
$\small\nabla_{\Thetav_t}=\{\nabla_{\Thetav_t^{(l)}}\}_{l=1}^L$   /
$\small\hn_{\Thetav_t}=\{\hn_{\Thetav_t^{(l)}}\}_{l=1}^L$ be the corresponding FP / AC gradient, defined as Eq.~(\ref{eqn:fp-bp}) / Eq.~(\ref{eqn:ac-bp}). Furthermore, let $\small\Ld(\Thetav)$ be the batch loss on the entire dataset. Then, both $\small\nabla_{\Thetav}$ and $\small\hn_{\Thetav}$ are stochastic estimators of the batch gradient $\small\nabla_{\Thetav}\Ld(\Thetav)$. The stochasticity of the FP gradient $\small\nabla_{\Thetav}$ comes solely from random sampling of the minibatch, which we assume to be unbiased, i.e., $\small\E{\nabla_{\Thetav}}=\nabla_{\Thetav}\Ld(\Thetav)$. On the other hand, the stochasticity of the AC gradient $\small\hn_{\Thetav}$ further comes from the random quantization of the context. 
The question is whether the AC gradient can be made unbiased as well, and for this, the answer is positive.

\begin{theorem}\label{thm:bias} (Unbiased  Gradient) There exists random quantization strategies for $\small\hat\Cv$, such that \vspace{-.5em}$$\small\E{\hn_{\Thetav}}=\nabla_{\Thetav}\Ld(\Thetav).$$
\end{theorem}

\vspace{-1em}
Intuitively, according to the chain rule, the back-propagation Eq.~(\ref{eqn:fp-bp}) can be rewritten as \vspace{-.5em}
{\footnotesize
\begin{align}
\nabla_{H^{(l-1)}_{ij}} = \sum_{kl} \tfrac{\partial H^{(l)}_{kl}}{\partial H^{(l-1)}_{ij}} \nabla_{H^{(l)}_{kl}},
\nabla_{\Theta^{(l)}_{i}} = \sum_{kl} \tfrac{\partial H^{(l)}_{kl}}{\partial \Theta^{(l)}_{i}} \nabla_{H^{(l)}_{kl}}.\label{eqn:linear-bp-2}
\end{align}
}%

\vspace{-1.25em}
Take 
$\footnotesize\hat\Cv(\Hl{l-1}, \Thetal{l})=Q(\{\sfrac{\partial H^{(l)}_{kl}}{\partial H^{(l-1)}_{ij}}, \sfrac{\partial H^{(l)}_{kl}}{\partial \Theta^{(l)}_{i}}\})$
, where $Q(\cdot)$ is an unbiased quantizer. As Eq.~(\ref{eqn:linear-bp-2}) is just a linear operation, we can show that $\small\hn_{\Hl{l-1}}$ and $\small\hn_{\Thetal{l}}$ are unbiased as long as $\small\hat\Cv(\Hl{l-1}, \Thetal{l})$ and $\small\hn_{\Hl{l}}$ are unbiased, which can be proven by induction. 
The linear layer $\small\hn_{\Thetal{l}} = Q(\Hl{l-1})^\top \hn_{\Hl{l}}$ is especially simple, where we can just use $\small Q(\Hl{l-1})$ as the compressed context. General layers are more complicated as directly storing the Jacobian matrices $\footnotesize\{\sfrac{\partial H^{(l)}_{kl}}{\partial H^{(l-1)}_{ij}}, \sfrac{\partial H^{(l)}_{kl}}{\partial \Theta^{(l)}_{i}}\}$ might be prohibitive. However, we show in Appendix~\ref{sec:layers} that for most frequently used layers, including convolution, pointwise, normalization, and up/down sampling, can be approximated in an unbiased way with a practical cost.



Given an unbiased gradient, we now establish the convergence of \method. Assume the SGD iteration takes the form $\small\Thetav_{t+1}\leftarrow \Thetav_{t}-\alpha\hn_{\Thetav_t}$, starting from an initial model $\small\Thetav_1$, and 

\noindent\textbf{A1.} The loss $\small\Lc_{\Dc}(\Thetav)$ is continuous differentiable and $\small\nabla\Lc_{\Dc}(\Thetav)$ is $\beta$-Lipschitz continuous. \\
\noindent\textbf{A2.} $\small\Lc_{\Dc}(\Thetav)$ is bounded below by $\Lc_{inf}$.\\
\noindent\textbf{A3.} There exists $\sigma^2>0$, such that  $\small\forall\Thetav$, $\small\Var{\hn_{\Thetav}}\le \sigma^2$, where for any vector $\xv$,  $\small\Var{\xv}:=\Eb\norm{\xv}^2 - \norm{\E{\xv}}^2$. 




The following convergence theorem is a standard result for SGD, taken from Theorem 4.8 in~\citet{bottou2018optimization}.
\begin{theorem}(Convergence)\label{thm:convergence} 
	If A1-A3 holds, and $0<\alpha\le \frac{1}{\beta}$, take the number of iterations $t$  uniformly from $\{1, \dots, T\}$, where $T$ is a maximum number of iterations. Then\vspace{-.5em}
	{\small
	\begin{align}
	\Eb\norm{\nabla\Lc_{\Dc}(\Thetav_t)}^2 \le \frac{2(\Lc(\Theta_1) - \Lc_{inf})}{\alpha T} + \alpha\beta\sigma^2.\label{eqn:thm-1}
	\end{align}
}%
\end{theorem}
Eq.~(\ref{eqn:thm-1}) is composed of two terms. The first term converges to zero as the number of iterations $T$ goes to infinity, while the second term does not. Intuitively, the algorithm converges to the neighborhood of a stationary point, where the radius is controlled by the gradient variance. Note that unlike the previous work~\cite{fu2020don}, the convergence of \method is established for general network architectures, not just for multi-layer perceptrons.


\subsection{Gradient Variance}
According to Thm.~\ref{thm:convergence}, gradient variance plays a critical role to the quality of the converged solution. We investigate how does the quantization affect the variance, so we can design quantization strategies accordingly.  
Let $\small\Gv_{\Hv}(\cdot)$ and $\small\Gv_{\Thetav}(\cdot)$ be components of $\small\Gv(\cdot)$, corresponding to $\small\nabla_{\Hv}$ and $\small\nabla_{\Thetav}$. For simplicity, let $\small\Cl{l}$ and $\small\hCl{l}$  be the full-precision and compressed context. Further, define 
{\footnotesize
\begin{align*} 
 \GT{l\sim m}{\hn_{\Hl{m}}, \hat\Cv^{(m)}} = \GT{l}{\GH{l+1}{\cdots
		\GH{m}{\hn_{\Hl{m}}, \hCl{m} }
		\cdots , \Cl{l+1}}, \Cl{l}},
\end{align*}
}%
which is the gradient $\small\hn_{\Thetal{l}}$ computed from $\small\hn_{\Hl{m}}$,  using the compressed context only at the $m$-th layer, and the full-precision context for all the other layers. Then, the gradient variance  is specified by the following theorem:


\begin{theorem} (Gradient Variance) \label{thm:grad-var}\vspace{-.5em}
%
{\small
\begin{align}
\Var{\hn_{\Thetal{l}}} = \Var{ \nabla_{\Thetal{l}}} +  \sum_{m=l}^{L} 
\E{
	\Varcond{ \GT{l\sim m}{
			\hn_{\Hl{m}}, \hat\Cv^{(m)}
	}}{\hn_{\Hl{m}}}
} .\label{eqn:grad-var}
\end{align}
}%
\end{theorem}
\vspace{-1em}
Thm.~\ref{thm:grad-var} disentangles all stochasticities in the gradient. The first term in Eq.~(\ref{eqn:grad-var}) is just the FP gradient variance, and it accounts for the  minibatch sampling. All the rest terms account for the noise of utilizing compressed context. Specifically, the term with $\footnotesize\GT{l\sim m}{\cdot, \hat\Cv^{(m)}}$ is the variance introduced by utilizing the compressed context $\small\hCl{m}$. See Appendix~\ref{sec:var-profile} for a visualization of these terms. 

The significance of Thm.~\ref{thm:grad-var} is in two folds. Firstly, it tells how much extra variance does activation compression introduce. If the activation compression variance is much smaller than the origin minibatch sampling variance, according to Thm.~\ref{thm:convergence}, we are confident that \method will converge similarly with FP training. In this case, we may reduce the numerical precision for free, as the quantization variance is negligible. Secondly, having an exact measurement of the variance, we can design quantization strategies to explicitly minimize it, as we shall see soon. 


\section{Compression Strategy}

\begin{figure}[t]
    \centering
\begingroup
\setlength{\tabcolsep}{0pt} 
\renewcommand{\arraystretch}{0} 
\scriptsize
\begin{tabular}{m{1cm}m{8cm}}
(a) &\includegraphics[width=.9\linewidth]{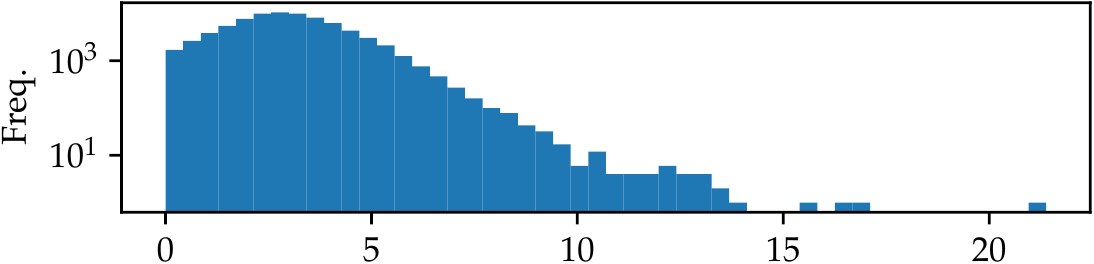}\\
(b) &\includegraphics[width=.9\linewidth]{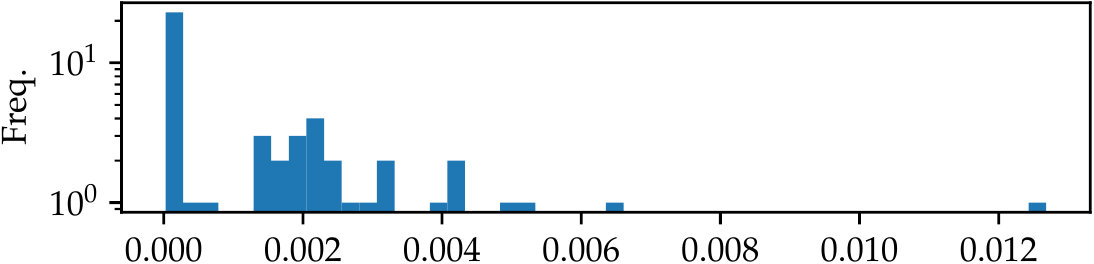}\\
(c) &\includegraphics[width=.9\linewidth]{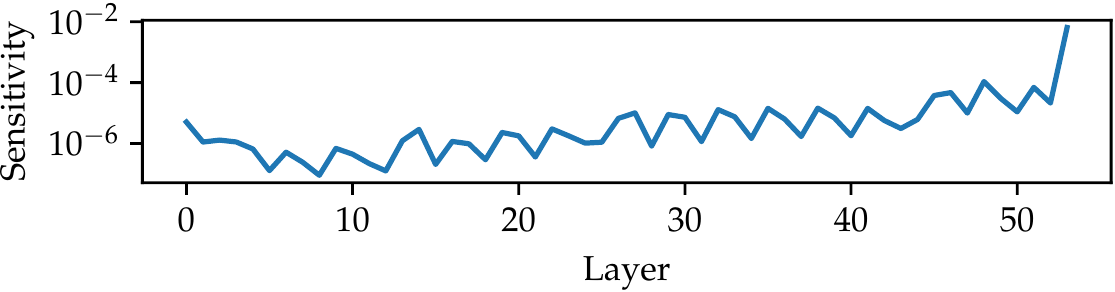}\\
\end{tabular}
\endgroup
\vspace{-1em}
\caption{\small{Heterogeneity in a ResNet50's activations.
(a) Histogram of the per-group range at the \texttt{conv\_2\_2\_1} layer;
(b) Histogram of the per-sample sensitivity at the same layer;
(c) The sensitivity per dimension for each layer.
}
    \label{fig:heterogeneity}}   \vspace{-1em}
\end{figure}

As mentioned earlier, \method compresses the activation by quantizing them to lower precision. As the number of bits goes down, the compression ratio gets better, but the gradient variance also grows. 

The activation is highly heterogeneous. As illustrated in Fig.~\ref{fig:heterogeneity}, the activation's magnitude, sensitivity, and dimensionality vary across different feature dimensions, samples in the batch, and network layers. Therefore, it is suboptimal to use the same quantization scheme for all the elements, as done in prior works~\cite{chakrabarti2019backprop,fu2020don}. We design \method 's quantization strategy to be aware of these heterogeneities. Based on our theory, \method  tunes its quantization strategy on-the-fly to approximately minimize the variance defined as Eq.~(\ref{eqn:grad-var}). As we mentioned in Sec.~\ref{sec:related}, these techniques exploit unique characteristics of the activation compression problem, and differ from existing methods for quantized training and model compression. 


\subsection{Per-group Quantization}\label{sec:pergroup}
First, we propose a per-group quantization strategy to tackle the distinct numerical range across feature dimensions. Given an activation tensor $\small\Hv\in \Rb^{N\times D}$, we partition its dimensions into groups $\hv_{ni}$, where each group has $G$ elements. 
The numbers are quantized to $b$-bit unsigned integers, or $\small B=2^b-1$ quantization bins. 
For each element, we compute the minimum and maximum, and scale the activation:\vspace{-.5em}
$$
\small
\bar\uv_{ni}\leftarrow B\left(\hv_{ni} - Z_{ni}\right)/R_{ni},
$$\vspace{-1.4em}
where $\small R_{ni}=\max\{\hv_{ni}\} - \min\{\hv_{ni}\}$ is the range, $\small Z_{ni}=\min\{\hv_{ni}\}$ is the zero point, and 
$\small\bar\uv_{ni}$ is the activation scaled to $[0, B]$. Convert $\small\bar\uv_{ni}$ to integers with stochastic rounding~\cite{courbariaux2015binaryconnect} and store the result in memory as\vspace{-.5em}
$$\small
\hat\uv_{ni} = \ceil{\bar\uv_{ni}} ~~\mathrm{w. prob. }~~\bar\uv_{ni}-\floor{\bar\uv_{ni}} 
~~\mathrm{otherwise}~~  \floor{\bar\uv_{ni}}.
$$
\vspace{-1em}
During back-propagation, the activation is dequantized as\vspace{-.5em}
$$\small
\hat\hv_{ni} = \hat\uv_{ni}R_{ni}/B + Z_{ni}.
$$
\vspace{-1em}
Due to the unbiased nature of stochastic rounding, it is clear that $\small \E{\hat\uv_{ni}}=\bar\uv_{ni}$ and $\small \E{\hat\hv_{ni}}=\hv_{ni}$. 

Assuming that $\small \bar\uv_{ni}-\floor{\bar\uv_{ni}} \sim \Uc(0, 1)$, the quantization variance is $\small \Var{\hat\hv_{ni}}=\frac{R_{ni}^2}{B^2}\Var{\hat\uv_{ni}}=\frac{R_{ni}^2G}{6B^2}$. The advantage of per-group quantization (PG) can be seen through the variance. Existing quantization strategies~\cite{chakrabarti2019backprop,fu2020don} use a single range and zero-point per tensor, which can be viewed as a single group with the range $R=\max_{ni} R_{ni}$. However, as illustrated in Fig.~\ref{fig:heterogeneity}(a), the range for most groups is far smaller than $R$. Therefore, this strategy uses unnecessarily large range for most groups, significantly enlarging the variance. 
In practice, we set $G=256$ and store the per-group range and zero points in \texttt{bfloat16}, so each group costs extra 32 bits, which is 0.125 bits on average.

\subsection{Fine-Grained Mixed-Precision}\label{sec:fine-grined}
To further reduce the variance, \method uses mixed-precision quantization strategies, that choose the numerical precision adaptively for each sample and each layer. Let $\small\bnl$ be the number of bits for sample $n$'s activation at layer $l$, $\small\Hv_n^{(l)}$. Let $\small\Bnl$ be the corresponding number of quantization bins. Theoretically, $\small\bnl$ should be chosen to minimize the quantization variance specified as the second term in Eq.~(\ref{eqn:grad-var}). However, the full gradient variance is too complicated to be tractable. Instead, \method minimizes the following objective\vspace{-.75em}
{\small
\begin{align}\label{eqn:mp-var}
\mbox{Var} = \sum_{l=1}^L \E{\Varcond{ \GT{l}{\hn_{\Hl{l}}, \hCl{l}} }{\hn_{\Hl{l}}}},
\end{align}
}\vspace{-1em}
which omits some terms from Eq.~(\ref{eqn:grad-var}). Firstly, it omits the minibatch sampling term, which is not affected by the quantization scheme. Secondly, it only keeps the impact to $\small\hn_{\Hl{l}}$ from $\small\hCl{l}$, omitting all the more distant contexts $\small\hCl{m} (m>l)$. As studied in Appendix~\ref{sec:var-profile}, the impact diminishes as the parameter and context become more distant. We find optimizing with this approximate objective already significantly reduces the true gradient variance.


Regarding the specific form of variance, we  take  linear layers as an example, where \\
$
\footnotesize\Varcond{ \GT{l}{\hn_{\Hl{l}}, \hCl{l}} }{\hn_{\Hl{l}}}=\Var{\hHlT{l-1} \hn_{\Hl{l}}}.
$ Simplifying the notations by omitting the conditional variance, layer indices, and let $\small\nabla:=\hn_{\Hl{l}}$, we have\vspace{-.75em}
%
{
\footnotesize
\begin{align}
\Var{\hat\Hv^\top \nabla}=\sum_{ij} \mathrm{Var}[\sum_n\hat h_{ni}\nabla_{nj}]=\sum_{ijn}\nabla_{nj}^2 \Var{\hat h_{ni}}
=\frac{G}{6}\sum_{ijn} \nabla_{nj}^2  R_{ni}^2 / B_n^2
= \frac{G}{6} \sum_n \norm{\nabla_n}^2 \norm{\Rv_n}^2 / B_n^2,\label{eqn:linear-var}
\end{align}
}\vspace{-1em}
where $G$ and $R_{ni}$ are the group size and per-group range defined in Sec.~\ref{sec:pergroup}, and $\Rv_n=\{R_{ni}\}$. For each sample, the variance depends on  the gradient magnitude $\small\norm{\nabla_n}^2$ and the range $\small\norm{\Rv_n}^2$. 

In general, we can minimize the overall variance under a bits budget $b_{total}$ by allocating more bits to sensitive layers and samples, described as the following optimization problem:\vspace{-.5em}
{\small
\begin{align}
\min_{b_n^{(l)}} \sum_{l=1}^L \sum_{n=1}^N w_n^{(l)} / {B_n^{(l)}}^2 ~~\mbox{s.t.}  \sum_{l=1}^L D^{(l)} \sum_{n=1}^N b_n^{(l)} \le b_{total},\label{eqn:allocation-problem}
\end{align}
}%
where $\small\Bnl=2^{\bnl}-1$ as defined earlier, $D^{(l)}$ is the feature dimensionality, and $\small w_n^{(l)}$ is the sensitivity for sample $n$ at layer $l$. For linear layers, we have $\small w_n^{(l)}=\frac{G}{6}\lVert\hn_{\hv_n^{(l)}}\rVert^2 \lVert\Rv_n^{(l)}\rVert^2$ by Eq.~(\ref{eqn:linear-var}). We derive the sensitivity for other  layers in Appendix~\ref{sec:layers}.

Mixed-precision can reduce the gradient variance significantly. According to Fig.~\ref{fig:heterogeneity}(b, c), the sensitivity is diverse across samples, and the per-dimension sensitivity varies by several orders of magnitudes across layers. Mixed-precision considers these heterogeneities. Furthermore, with mixed-precision, the average number of bits is no longer limited to integers, enabling a more fine-grained tradeoff between compression ratio and gradient variance. 


\subsection{Run-time Adaptation}\label{sec:runtime}
To best utilize the data characteristics, \method tunes the mixed-precision quantization strategy at run time. In each SGD iteration, the tuning happens in two stages:

\vspace{-.25em}
1. \emph{(per-sample allocation)}
During the forward propagation for each layer $l$, \method computes and stores the sensitivity $\small w_n^{(l)}$ for each sample. Then, it computes the optimal $\bnl$ for each sample under a fixed bits budget $b^{(l)}$ \emph{for this layer}, by solving Prob.~(\ref{eqn:allocation-problem}). $\bnl$ is used to compress the activation. 

2. \emph{(per-layer allocation)} After finishing the back propagation, \method solves Prob.~(\ref{eqn:allocation-problem}) again for all the layers together, and sets $b^{(l)}\leftarrow \sum_n b_n^{(l)}$. 


\vspace{-.25em}
Prob.~(\ref{eqn:allocation-problem}) is a discrete optimization problem, and can be solved exactly by dynamic programming (DP). However, DP is too costly to be computed in each SGD iteration. Instead, \method adopts a simple greedy algorithm. It starts with high numerical precision, e.g., $b_n^{(l)}=8$ for all layers and samples, and progressively reduces the precision until it fits in the total bits budget. In each move, it chooses a $b_n^{(l)}$ to reduce by one, such that the increment of variance  is minimal. With a binary heap for picking up the optimal move, this greedy algorithm runs in  $O(NL\log_2(NL))$, where $N$ is the batch size, and $L$ is the model depth.

\vspace{-.25em}
Finally, the sensitivity $w_n^{(l)}$ might depend on the gradient magnitude for each sample, which is unknown by the time we compress. \method provides two options for estimating the gradient magnitude. The first option uses the stale gradient magnitude in the last epoch. The second option uses the moving average of gradient magnitude across samples. Both strategies work well in practice. 
\section{System Implementation}

We implement \method as a library based on PyTorch \cite{paszke2019pytorch}.
The system includes a collection of activation compressed layers.
Using the system only requires substituting all PyTorch's default layers with \method's layers (e.g., replace all \texttt{torch.nn.Conv2d} with \texttt{actnn.Conv2d}).
This substitution can be done automatically with a model converter.
The system also provides different optimization levels to control the trade-off between memory and speed.

\subsection{Activation Compressed Layers}

\begin{figure}[t]
	\centering
	\includegraphics[width=0.9\mywidth]{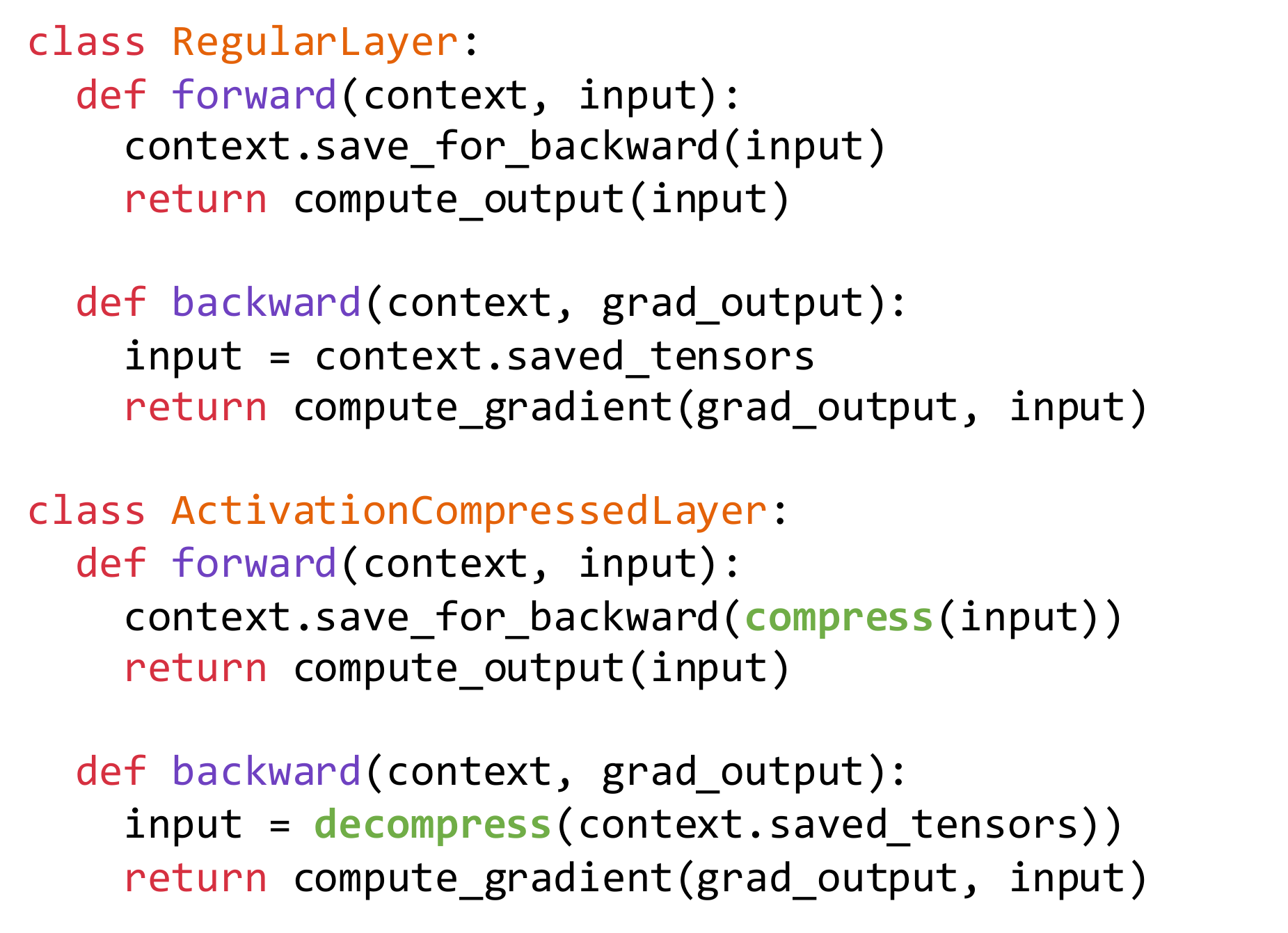}
	\vspace{-1em}
	\caption{\small Pseudo code for activation compressed layers.} 
	\vspace{-1em}
\label{fig:act_layers}
\end{figure}

Fig.~\ref{fig:act_layers} shows the pseudo-code for regular layers and activation compressed layers.
A regular layer saves full-precision activations into its context, while an activation compressed layer compresses the activations before saving them. Therefore, the memory required for saving context is reduced.
We implement highly-optimized CUDA kernels for compression and decompression, which
quantize floating-point numbers into integers and compress them into bit streams.

\begin{table}[t]
	\centering
	\caption{\small Usability Comparison of Memory Saving Systems. The systems include Checkmate~\cite{jain2019checkmate}, DTR~\cite{kirisame2020dynamic},
		MONet~\cite{shah2020memory},
		SwapAdvisor~\cite{huang2020swapadvisor}, Capuchin~\cite{peng2020capuchin},
		and BLPA~\cite{chakrabarti2019backprop}.  Remat.=Rematerialization, Comp.=Compression, AOT=ahead-of-training.
	}
	\label{tab:system_usability}
	\resizebox{\mywidth}{!}{
		\setlength{\tabcolsep}{3pt}
		\begin{tabular}{cccccc}
			\toprule
			\multirow{2}{*}{System}   & \multirow{2}{*}{Method}            & \footnotesize{Arbitrary}  & \footnotesize{Dynamic}  & \footnotesize{Zero AOT}  & \footnotesize{Standalone} \\
			&             &  \footnotesize{Graph} &  \footnotesize{Exec.} &  \footnotesize{Overhead} &  \footnotesize{Package} \\
			\midrule
			Checkmate   & Remat.    & yes  & no   & no   & yes  \\
			MONet       & Remat.    & yes  & no   & no   & yes  \\
			DTR         & Remat.    & yes  & yes  & yes  & no   \\
			SwapAdvisor & Swap      & yes  & no   & no   & no   \\
			Capuchin    & Swap      & yes  & yes  & yes  & no   \\
			BLPA        & Comp.     & no   & no   & yes  & yes  \\
			\method     & Comp.     & yes  & yes  & yes  & yes  \\
			\bottomrule
		\end{tabular}
	}
\end{table}

\begin{table}[t]
	\centering
	\caption{\small Optimization levels for \method. \label{tab:opt_level}}
	\resizebox{0.9\mywidth}{!}{
		\begin{tabular}{ccc}
			\toprule
			Level &  Compression Strategy & Bits\\
			\midrule
			L0  & Do not compress & 32\\
			L1  & per-group quantization for conv. layers & 4, 32\\
			L2  & per-group quantization & 4 \\
			L3  & L2 + fine-grained mixed-precision & 2\\
			L4  & L3 + swapping & 2\\
			L5  & L4 + defragmentation & 2\\
			\bottomrule
	\end{tabular}}\vspace{-1em}
\end{table}

Tab.~\ref{tab:system_usability} compares the usability of \method against other memory saving systems.
\method's layers can be easily plugged into existing deep learning frameworks without modifying the frameworks themselves.
The layers are fully compatible with existing features of PyTorch, such as dynamic execution and auto-differentiation.
In contrast, advanced swapping and tensor rematerialization methods require heavy modification of the frameworks.
Another benefit of \method is not introducing any ahead-of-training overhead. In contrast, some systems (e.g., Checkmate, MONeT, SwapAdvisor) require non-trivial ahead-of-training overhead to solve expensive optimization problems, which can take up to hours.

\subsection{Optimization Levels}
There is a trade-off between memory saving and training speed.
More overhead will be introduced for saving more memory.
To exploit this trade-off, \method provides 6 optimization levels.
Higher levels can save more memory but with more overhead.

Tab.~\ref{tab:opt_level} lists these optimization levels. 
L1 uses per-group quantization to compress convolutional layers and leaves all other layers unchanged, while L2 uses per-group quantization for all the layers.
L3 further adds fine-grained mixed-precision, which achieves a better compression ratio with some additional computational overhead.
L4 combines compression with swapping. We swap out all compressed activations to CPU memory during the forward pass and swap them in during the backward pass. 
At L5, we improve the memory allocator to reduce fragmentation. PyTorch uses a caching allocator to reuse GPU buffers. This makes allocation faster but introduces a serious memory fragmentation issue. At this level, we disable this caching allocator for large tensors to reduce fragmentation.
\section{Experiments}

\begin{figure*}[t]
	\centering
	\begingroup
	\setlength{\tabcolsep}{0pt} 
	\scriptsize
	\begin{tabular}{m{3.8cm}m{3.8cm}m{3.8cm}m{3.8cm}m{1.8cm}}
		\includegraphics[width=.95\linewidth]{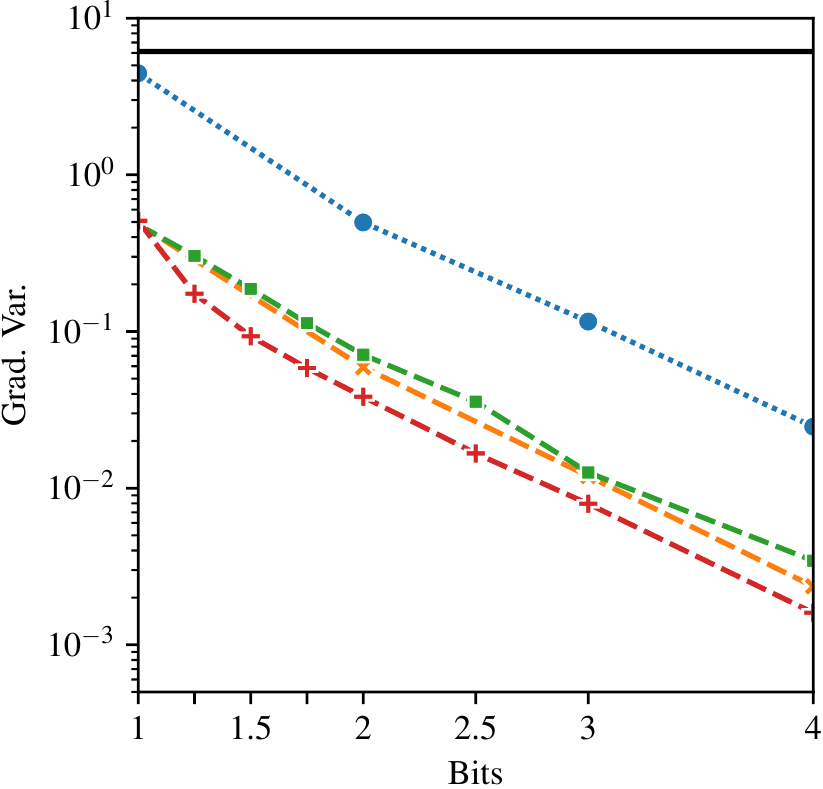}&
		\includegraphics[width=.95\linewidth]{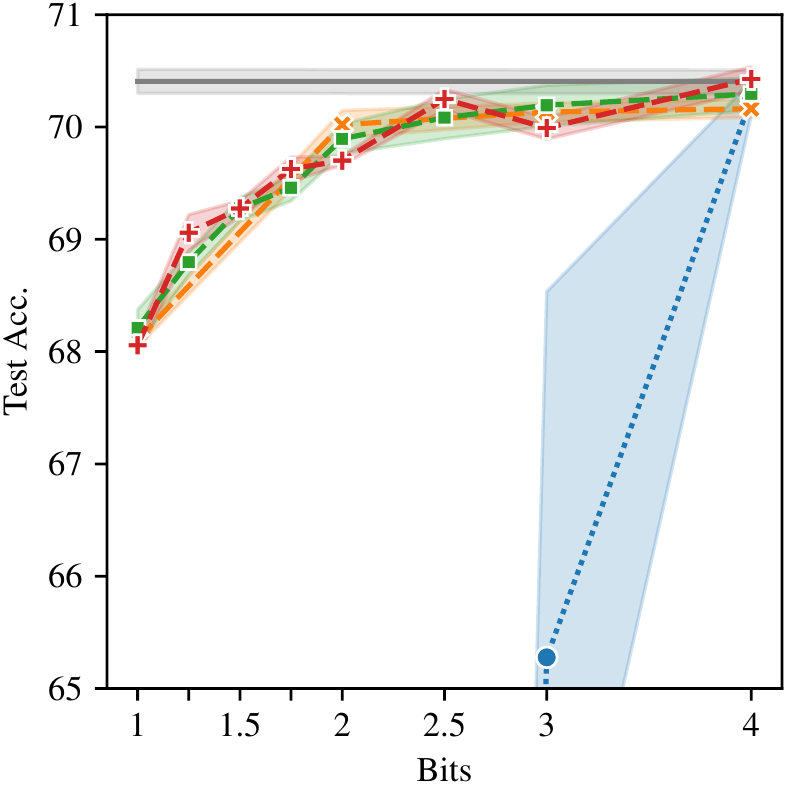}&
		\includegraphics[width=.95\linewidth]{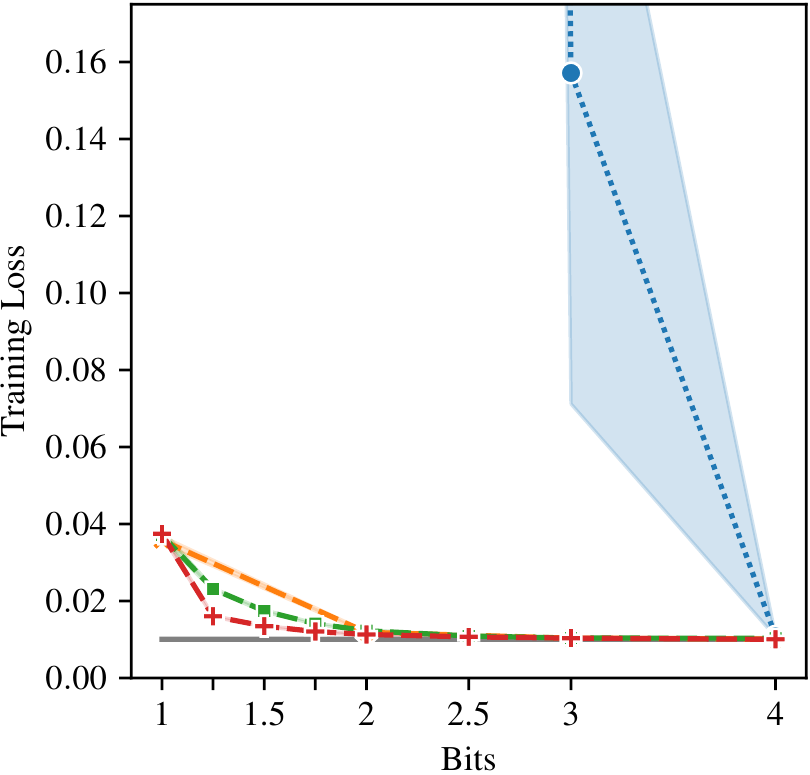}&
		\includegraphics[width=.95\linewidth]{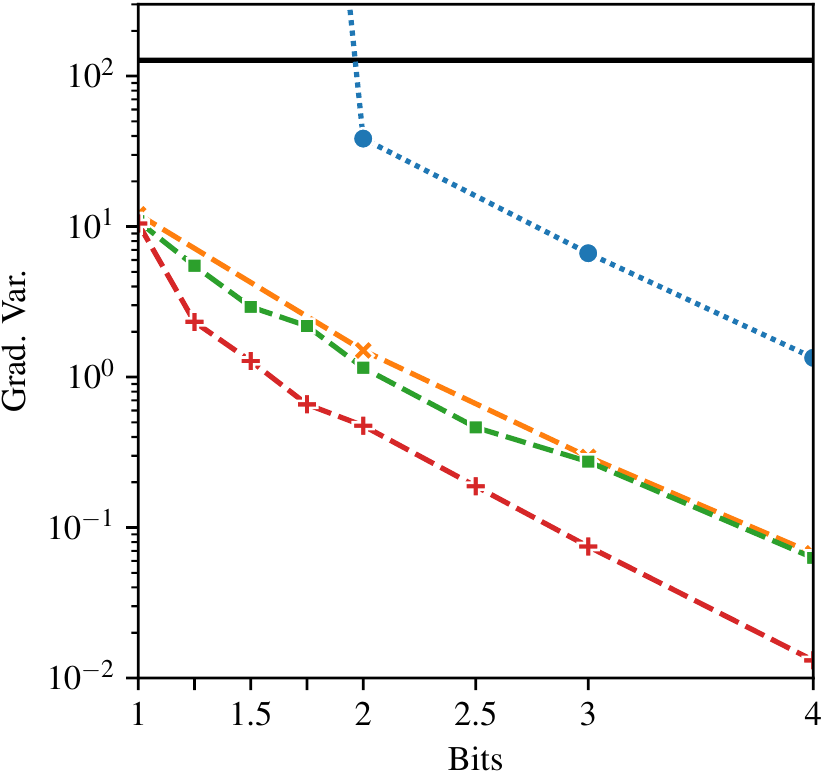}&
		\includegraphics[width=\linewidth]{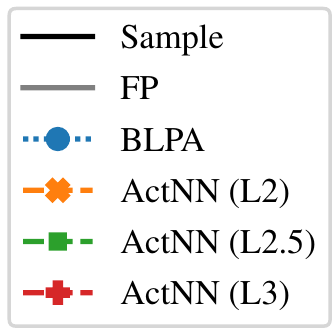}\\
		(a) Gradient variance on CIFAR-100 & (b) Testing accuracy on CIFAR-100 & (c) Testing loss on CIFAR-100 & (d) Gradient variance on ImageNet &
	\end{tabular}
	\endgroup
	\caption{\small Ablation study on the quantization strategy. BLPA diverges with 1 and 2 bits. The gradient variance is calculated at the 10th epoch for CIFAR-100 and the 50th epoch for ImageNet.  Sample=minibatch sampling, FP=full precision.}
	\label{fig:gradvar}
\end{figure*}

In this section, we evaluate \method on a wide range of tasks and compare it with other memory-saving systems. We use open-source model implementations and recipes for all tasks. Detailed experimental setup is in Appendix~\ref{sec:exp-setup}. The training logs for all the models we used are available at \url{https://wandb.ai/actnn}.

\subsection{Quantization Strategy}
We study the impact of activation compression on the accuracy and demonstrate our heterogeneity-aware quantization strategies. The baselines are full-precision (FP) training and BLPA~\cite{chakrabarti2019backprop}, a per-tensor quantization strategy with fixed numerical precision. To compare quantization strategies, we use \method (L2, L3) listed in Tab.~\ref{tab:opt_level}. \method (L4, L5) do not further compress the activation, so they have identical behavior with \method (L3). We also add an \method (L2.5) for comparison, which only allocates bits between samples while keeping all the layers with the same bits per dimension.

We perform the ablation studies on ResNet-56~\cite{he2016identity} on CIFAR-100~\cite{krizhevsky2009learning}, and ResNet-50~\cite{he2016deep} on ImageNet~\cite{deng2009imagenet}. We also provide results on CIFAR-10 in Appendix~\ref{sec:cifar10} for reference. 
The average number of bits is varied between $\{1, 1.25, 1.5, 1.75, 2, 2.5, 3, 4\}$. 
Non-integer bits is only supported by mixed-precision approaches, namely, \method (L2.5, L3). Each configuration is repeated by 5 times on CIFAR-100, and by once on ImageNet. Fig.~\ref{fig:gradvar}(a,d) shows the gradient variance from both minibatch sampling (``Sample'') and activation compression. The activation compression variance increases exponentially as the number of bits decreases, as analyzed in Eq.~(\ref{eqn:linear-var}). However, better algorithms achieve lower variance under a given bits budget. On ImageNet, the required bits to reach a variance below 2 is 4-bit for BLPA, 2-bit for \method (L2), 1.75-bit for \method (L2.5), and 1.5-bit for \method (L3). At the extreme 1-bit case, mixed-precision \method (L2.5, L3) fall back to \method (L2), since each number needs at least 1 bit to encode. This can be potentially improved by allowing the bits to go below 1 bit, e.g., with product quantization~\cite{stock2020and}, which we leave as future work. 

The validation accuracy (Fig.~\ref{fig:gradvar}(b)) and the training loss (Fig.~\ref{fig:gradvar}(c)) align with the gradient variance results, where BLPA needs 4-bit to achieve near-lossless accuracy, and \method only needs 2-bit. These results support our theoretical analysis (Thm.~\ref{thm:convergence}), where the gradient variance can be used as a surrogate of the approximation quality. Tab.~\ref{tab:imagenet} shows the results on ImageNet. \method significantly outperforms BLPA. BLPA achieves near-lossless result at 4-bit, and diverges at 3-bit, while \method (L2.5, L3) are still lossless at 2-bit. The result of BLPA at 4-bit is on par with \method (L3) with only 1.5 bits. 
Remarkably, \method converges and gives reasonable results even at the extreme 1.25-bit setting. This shows the robustness of our quantization strategy. Another relevant work, TinyScript~\cite{fu2020don}, utilizes non-uniform quantization intervals for activation compression. TinyScript has no public code, and it reports 7.74\% top-5 error with 3-bit, which is on par with \method with 1.25-bit. 

As we mentioned in Sec.~\ref{sec:related}, \method can be combined with other quantized training approaches. Here, we demonstrate this idea by combining \method with a mixed-precision training library, AMP~\cite{amp}. In this setting, AMP accelerates the training by computing the convolutions using 16-bit floating point numbers. We further apply \method upon AMP to compress the saved activation, reducing the memory footprint. The result is shown in Table~\ref{tab:amp}. \method works well with AMP, and the accuracy is not affected. 

We also conduct an ablation study of the gradient estimation strategy discussed in Sec.~\ref{sec:runtime}, which is used for \method (L3). The result is shown in Fig.~\ref{fig:ablation_grad}, where ``Stale'' estimates the gradient magnitude with stale gradients, and ``Moving Average'' averages the gradient across training samples. Both strategies work well in practice, and there is not perceivable differences. We use the ``Moving Average'' strategy in all our experiments for its simplicity.

\begin{figure}[t]
    \centering
    \includegraphics[width=.8\mywidth]{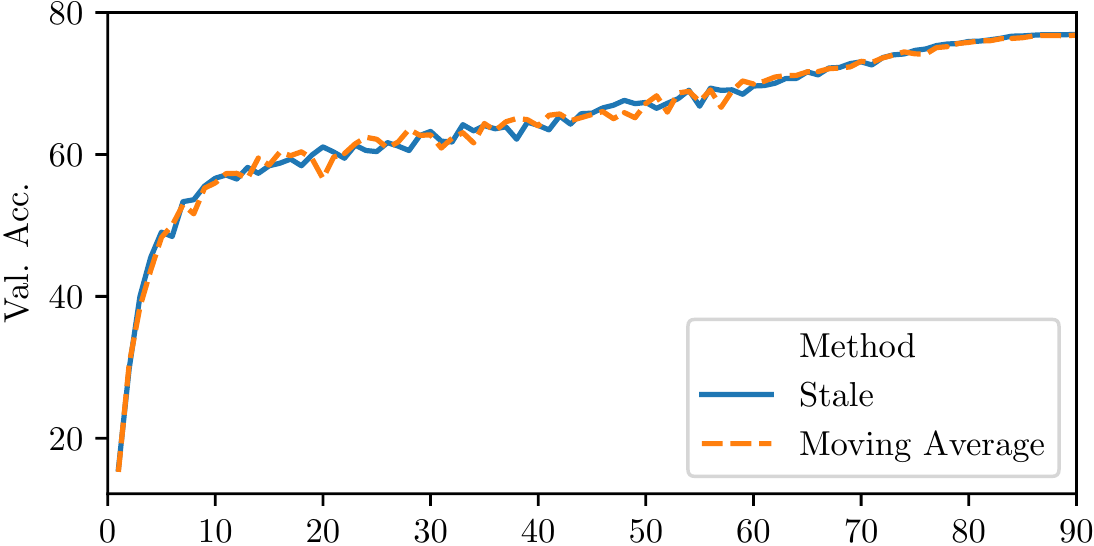}
    \caption{Ablation study on the gradient  estimation strategy.}
    \label{fig:ablation_grad}
    \vspace{-1em}
\end{figure}



\begin{table}[t]
	\centering
	\caption{\small ResNet-50 on ImageNet. N/A: not available; Div.: diverge; ``-'': skipped since lower precision achieves lossless results.\label{tab:imagenet}}
	\resizebox{0.9\mywidth}{!}{
		\begin{tabular}{ccccccc}
			\toprule
			Bits     &  32 & 4 & 3 & 2 & 1.5 & 1.25 \\
			\midrule
			FP     &  \textbf{77.1} &  N/A & N/A & N/A & N/A & N/A \\
			BLPA & N/A &  \textbf{76.6} & Div. & Div. & N/A & N/A \\
			\midrule
			\method (L2) & N/A & - & \textbf{77.4} & 0.1 & N/A & N/A \\
			\method (L2.5) & N/A & - & - & \textbf{77.1} & 75.9 & 75.1 \\
			\method (L3) & N/A & - & - & \textbf{76.9} & 76.4 & 75.9 \\
			\bottomrule
	\end{tabular}}
\end{table}

\begin{table}[t]
	\centering
	\caption{\small \method with mixed precision training. Act. Bits indicates the average number of bits for the saved activation.\label{tab:amp}}
	\small
		\begin{tabular}{ccc}
			\toprule
			Method & Act. Bits & Val. Acc. \\
			\midrule
			FP     & 32  & 77.1 \\
			AMP     & 16 & 77.3 \\
			\method (L2) + AMP & 4 & 77.1 \\
			\method (L3) + AMP & 2 & 76.9 \\
			\bottomrule
	\end{tabular}
\end{table}

\subsection{Memory Saving and Computational Overhead}

\begin{table}[t]
	\centering
	\caption{\small Memory usage before backward pass. ``Total Mem.'' includes model parameters, optimizer states, input data and activation. ``Act. Mem.'' includes activation for all layers except the final loss layer. R=memory saving ratio; OOM=out of memory.}
	\label{tab:memory_analysis}
	\resizebox{\mywidth}{!}{
		\begin{tabular}{c|c|ccc|ccc}
			\toprule
			\multirow{2}{*}{Network} & \multirow{2}{*}{Batch} & \multicolumn{3}{c|}{Total Mem. (GB)} & \multicolumn{3}{c}{Act. Mem. (GB)} \\
			&                     & FP & \method(L3) & R & FP & \method (L3) & R\\
			\midrule
			\multirow{4}{*}{ResNet-152} & 32    & 6.01   & 1.18 & 5$\times$  & 5.28    & 0.44 & 12$\times$\\
			& 64    & 11.32  & 1.64 & 7$\times$  & 10.57   & 0.88 & 12$\times$\\
			& 96    & OOM    & 2.11 & / & OOM     & 1.32 & /\\
			& 512   & OOM    & 8.27 & / & OOM     & 7.01 & / \\
			\midrule
			\multirow{4}{*}{FCN-HR-48}       & 2   & 5.76   & 1.39 & 4$\times$  & 4.76  & 0.39 & 12$\times$ \\
			& 4   & 10.52  & 1.79 & 6$\times$  & 9.52  & 0.79 & 12$\times$\\
			& 6   & OOM    & 2.17 & / & OOM   & 1.18  & /\\
			& 20  & OOM    & 4.91 & / & OOM   & 3.91  & / \\
			\bottomrule
		\end{tabular}
	}
\end{table}

\begin{figure}[t]
	\centering
	\begingroup
	\setlength{\tabcolsep}{0pt} 
	\renewcommand{\arraystretch}{1} 
	\scriptsize
	\begin{tabular}{m{1.5cm}m{8cm}}
		\begin{minipage}{\linewidth}(a)\\ResNet-50\end{minipage} &\includegraphics[width=.85\linewidth]{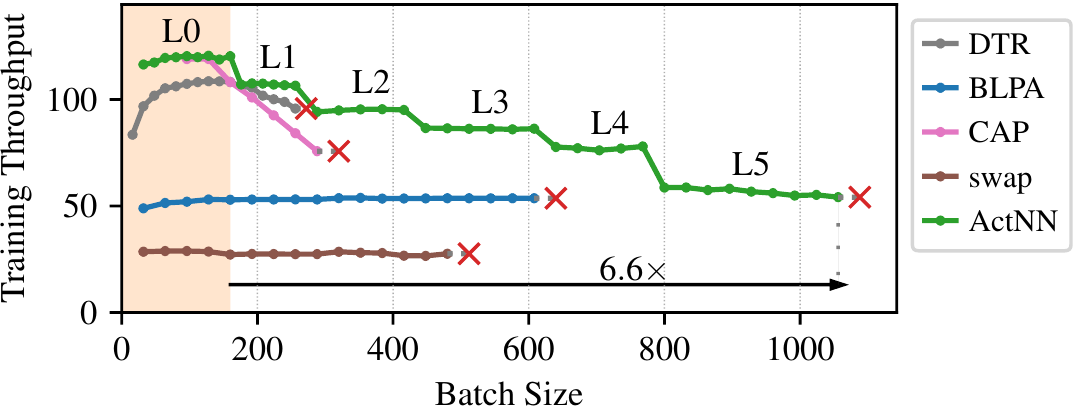}\\
		\begin{minipage}{\linewidth}(b)\\WideResNet-101\end{minipage}
		&\includegraphics[width=.85\linewidth]{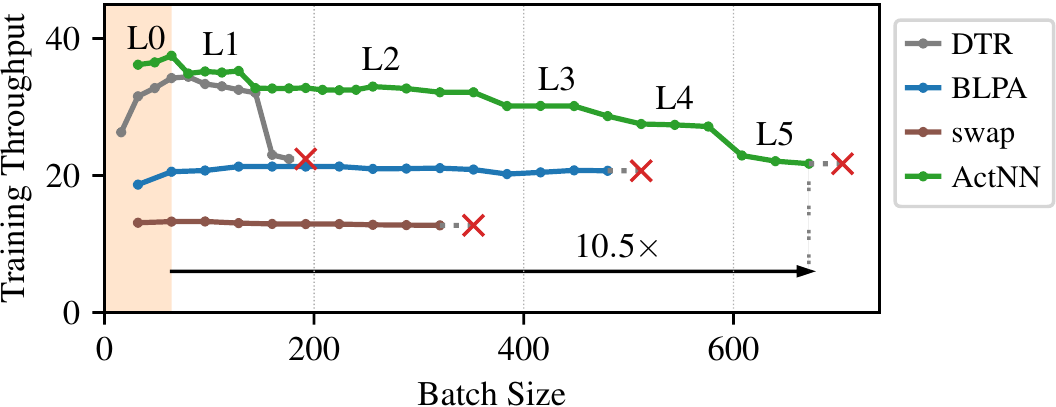}\\
		\begin{minipage}{\linewidth}(c)\\DenseNet-201\end{minipage}
		&\includegraphics[width=.85\linewidth]{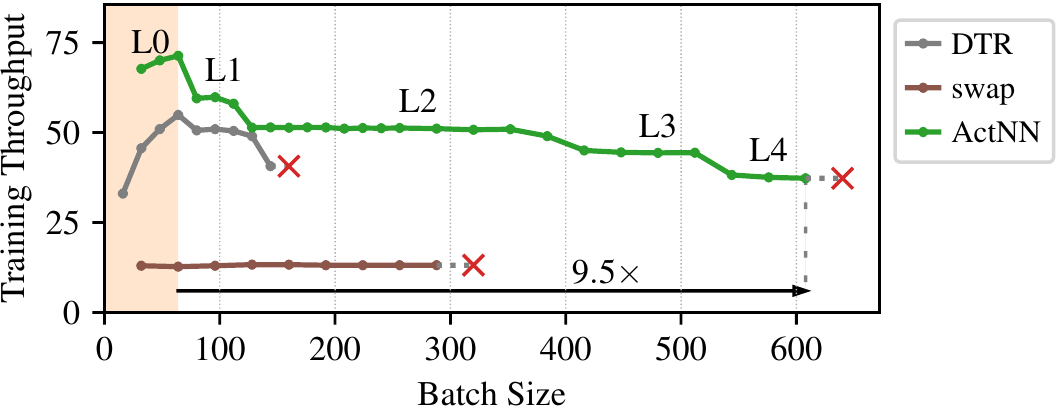}\\
	\end{tabular}
	\endgroup
\caption{Training throughput vs batch size. Red cross mark means out-of-memory. The shaded yellow region denotes the possible batch sizes with full precision training given the memory budget.
}
\label{fig:throughput_batch_size}
\end{figure}

Next, we measure the memory saving and the overhead introduced by \method.
The experiments are done with PyTorch v1.7 and an AWS \texttt{g4dn.4xlarge} instance, which has a 16GB NVIDIA T4 GPU and 64GB CPU memory.


Tab.~\ref{tab:memory_analysis} shows the memory usage right before backward pass. The activation memory reaches its peak at this point.
For both models, activations consume over 90\% of the total memory. This justifies the potential of activation compressed training.
\method compresses activation memory by $12\times$. This matches the theoretical value.
Take a Conv-BN-ReLU block as example, FP takes 32 bits (Conv) + 32 bits (BN) = 64 bits, while \method takes 2.125 bits (Conv) + 2.125 bits (BN) + 1 bit (ReLU) = 5.25 bit. 
The extra 0.125 bit is used to store the zero point and scale for each group.
The compression ratio is $64/5.25\approx 12$.

We compare the training throughput of \method against other memory saving systems in Fig.~\ref{fig:intro} and Fig. ~\ref{fig:throughput_batch_size}.
Each curve shows the trade-off between memory saving and training speed for one system.
``DTR'' is dynamic tensor rematerialization ~\cite{kirisame2020dynamic}, a state-of-the-art rematerialization method for dynamic graphs. DTR runs out-of-memory very soon when trying to increase the batch size.
``BLPA'' is the system implemented in \citet{chakrabarti2019backprop}. It only supports a dedicated ResNet architecture, so we cannot run it on DenseNet. Further, its dedicated design is not compatible with fused high-performance kernel libraries, so its training speed is very slow.
``CAP'' is the Capuchin system based on swapping and recomputation ~\cite{peng2020capuchin}. It is not open-source, so we get relative overhead numbers from its paper and compute the scaled training throughput. Because our system runs in dynamic graph mode, we use Capuchin's numbers in eager mode as a fair comparison.
``swap'' is a simple swapping strategy that swaps all activations to the CPU. However, the CPU memory is finite, so it cannot increase the batch size unlimitedly.
``\method'' is our system with 6 different levels of optimization, so the curve of \method shows roughly 6 segments. As shown in the figures, \method achieves a much larger maximum batch size and extends the Pareto frontier significantly over state-of-the-art systems.
\method enables training with a $6.6\times - 14.0\times$ larger batch size under the same memory budget.

\begin{table}[]
	\centering
	\caption{Comparison of the largest models \method can train before out-of-memory with the same batch size (64). D = depth = the number of layers, W = width = the base width of the bottleneck block, R = resolution = width and height of input images.}
	\label{tab:max_model}
	\resizebox{\mywidth}{!}{
		\begin{tabular}{c|ccc|ccc}
			\toprule
			\multirow{2}{*}{Dim.}  &  \multicolumn{3}{c|}{Maximum Value}  & \multicolumn{3}{c}{Training Throughput  (TFLOPS)}  \\
			& FP & \method (L3)  & \method (L4) & FP & \method (L3)  & \method (L4) \\
			\midrule
			D  &  160  & 660  & 1016  &   0.59   & 0.46  & 0.38  \\
			W  &  92   & 332  & 340   &   0.70   & 1.07  & 1.09  \\
			R  &  240  & 636  & 740   &   0.59   & 0.46  & 0.42   \\
			\bottomrule
	\end{tabular}}
\end{table}

To show the potential of training larger models, we scale a ResNet-152 to deeper, wider, or higher resolution.
\autoref{tab:max_model} compares the largest model we can train against full precision.
With the same memory budget and batch size (64), ActNN can scale a ResNet-152 to $6.4 \times$ deeper,
$3.7 \times$ wider or $3.1 \times $ higher resolution, while maintaining $64 \% - 155 \%$ original training throughput.

\subsection{Segmentation and Detection}
Here, we report results for high-resolution segmentation and detection in Tab.~\ref{tab:segmentation}. 
Activation memory is a more severe problem for these tasks, as the activation memory scales quadratically with the image resolution. 
For all the models, \method converges within $0.5\%$ mIoU / AP comparing with the full-precision baseline.
It worth noticing that \method could finish the default training recipe, batch size of 16/8 for detection/segmentation, within only one GPU.
Remarkably, by training with a larger batch size of 160, \method gains $1.8\%$/$1.4\%$ mIoU for Dilation8 (dilated FCN)~\cite{shelhamer2017fully} / FPN~\cite{kirillov2019panoptic}. 
This gain comes from the more reliable estimation of normalization parameters~\cite{peng2018megdet} with large batch size. 
Without memory saving techniques, such a large batch size is only achievable with a cluster of machines. 

\begin{table}
\caption{
Semantic segmentation on Cityscapes and object detection on Coco.  Models: HRNetV2W48 (HRNet)~\cite{wang2020deep}, 
ResNet-50 dilation 8 (Dilation8), FPN~\cite{kirillov2019panoptic}, RetinaNet~\cite{lin2017focal}. 
}
\label{tab:segmentation}
\vskip 0.15in
\begin{center}
\resizebox{0.9\mywidth}{!}{
\begin{tabular}{lllccccc}
\toprule
Task & Model & Method & Bits &  Batch &  mIoU / AP \\
\midrule
\multirow{9}{*}{Seg.} & HRNet & FP & 32 &  8 & 80.65 \\
& HRNet & \method (L3)  &2 & 8 &  81.02 \\
& HRNet & \method (L3) &2 &  160 &  80.31 \\
\cmidrule{2-6}
& Dilation8 & FP& 32 & 8 &  73.61 \\
& Dilation8 & \method (L3)&2 & 8 &  72.92 \\
& Dilation8 & \method (L3)&2 & 160 &  75.85 \\
\cmidrule{2-6}
& FPN & FP & 32& 8 &  74.52 \\
& FPN & \method (L3)&2 & 8 &  74.18 \\
& FPN & \method (L3)&2 & 160 &  75.98 \\
\midrule
\multirow{3}{*}{Det.} & RetinaNet & FP& 32 & 16 & 36.5\\
& RetinaNet & \method (L3)&2 & 16 & 36.2 \\
& RetinaNet & \method (L3)&2 & 80 & 36.0  \\
\bottomrule
\end{tabular}
}
\end{center}
\vskip -0.1in
\end{table}

\subsection{Self-supervised Learning}
Here, we test \method for two self-supervised learning methods, MoCov2~\cite{chen2020mocov2} and BYOL~\cite{grill2020bootstrap}. As the contrastive loss used in these methods involves comparing pairs of examples in a batch, larger batch size gives more accurate estimation of the contrastive loss, and has positive impact to the quality of the learned representation~\cite{chen2020simple}.

We directly apply \method (L3) to MoCov2 and BYOL. Both methods use ResNet-50 as the backbone. MoCov2~\cite{chen2020mocov2} is trained for 200 epochs and it uses the last layer's feature after global pooling for evaluation; BYOL~\cite{grill2020bootstrap} is trained for 300 epochs and it combines multiple layers' features for evaluation. As shown in Table~\ref{tab:ssl}, \method can train the models with significantly larger batch size per GPU, and it achieves good validation accuracy using only 2-bit activations. 
\begin{table}
\caption{
Self-supervised learning on ImageNet. 
}
\label{tab:ssl}
\vskip 0.15in
\begin{center}
\resizebox{0.9\mywidth}{!}{
\begin{tabular}{llcccccc}
\toprule
Model & Method & Bits &  Batch & GPUs &  Val. Acc. \\
\midrule
MoCov2 & FP & 32 & 256 & 8 & 67.69 \\
MoCov2 & \method (L3) & 2 & 512 & 2 & 67.25 \\
BYOL & FP & 32 & 256 & 8 & 72.35 \\ 
BYOL & \method (L3) & 2 & 1024 & 8 & 72.65 \\ 
\bottomrule
\end{tabular}
}
\end{center}
\vskip -0.1in
\end{table}

\section{Conclusions}
We have presented \method, a framework for training neural networks with randomly quantized activations. 
\method is grounded by the convergence guarantee for general network architectures that we provide. 
Quantization affects the convergence through the gradient variance. 
We propose per-group quantization and fine-grained mixed-precision quantization strategies, which approximately minimizes the gradient variance during training. 
On a wide range of tasks, \method achieves negligible accuracy loss with 2-bit activation, improving significantly over prior state-of-the-arts. 
\method can be readily applied as a collection of layers in PyTorch, and it enables up to $14\times$ batch size, $6.4\times$ deeper, $3.7\times$ wider, or $3.1\times$ higher-resolution models. 

\section*{Acknowledgements}
This work was supported by a gracious fund from Intel corporation, Berkeley Deep Drive (BDD), and Berkeley AI Research (BAIR) sponsors. In addition to NSF CISE Expeditions Award CCF-1730628, this research is supported by gifts from Amazon Web Services, Ant Group, Ericsson, Facebook, Futurewei, Google, Intel, Microsoft, Nvidia, Scotiabank, Splunk and VMware. We would like to thank the Intel VLAB team for providing us with access to their computing cluster. We also gratefully acknowledge the support of NVIDIA Corporation for their donation of two Titan Xp GPU used for this research. We would also like to acknowledge the UC Berkeley CLTC, ARO, DARPA, IARPA, NSF, and ONR for providing partial support of this work. Our conclusions do not necessarily reflect the position or the policy of our sponsors, and no official endorsement should be inferred.

\clearpage
\bibliography{refs}

\ifisarxiv
\bibliographystyle{unsrt}
\fi

\newpage
\appendix
\onecolumn
\section{Proof of the Theorems}\label{sec:theorems}
\subsection{Table of Notations}
\begin{table}[h]
\centering
\caption{Table of Notations.}
\begin{tabular}{cc}
\toprule
\textbf{Notation}     &  \textbf{Description} \\
\midrule
$\Xv$     & A batch of inputs (each row is a sample) \\
$\Yv$     & A batch of labels (each row is a sample) \\
$\Bc$     & A batch $\Bc=(\Xv, \Yv)$ \\
$N, L$    & Batch size, number of classes, and number of layers \\
$\Fv^{(l)}(\cdot; \Thetav^{(l)})$ & Forward function of the $l$-th layer with parameter $\Thetal{l}$ \\
$\Gv^{(l)}(\cdot; \cdot)$ & Backward function of the $l$-th layer \\
$\Cv(\Hl{l-1}, \Thetal{l}), \Cv^{(l)}$ & $l$-th layer's context \\
$\Cv(\Hl{l-1}, \Thetal{l}), \hat\Cv^{(l)}$ & $l$-th layer's compressed context \\
$\Lc=\ell(\Hl{L}, \Yv)$ & Minibatch loss function of prediction $\Hl{L}$ and label $\Yv$. \\
$\Lc_{\Dc}$ & Batch loss on the entire dataset. \\
$\nabla_{\Thetav} \Lc_{\Dc}$ & Batch gradient \\
$\nabla_{\Hl{l}}, \nabla_{\Thetal{l}}$ & Full-precision gradient of activation / parameter \\
$\hn_{\Hl{l}}, \hn_{\Thetal{l}}$ & Activation-compressed gradient of activation / parameter \\
$b_n^{(l)}, B_n^{(l)}$ & Number of quantization bits / bins for $\hv_n^{(l)}$\\
$G$ & Group size for per-group quantization \\
$R, R_{ni}, \Rv$ & Quantization range\\
\bottomrule
\end{tabular}
\label{tab:notation}
\end{table}


\subsection{Theorem~\ref{thm:bias}}
The FP gradient is defined by the recursion
\begin{align*}
\nabla_{\Hl{l-1}}, \nabla_{\Thetal{l}}
= \G{l}{\nabla_{\Hl{l}}, \Cv(\Hl{l-1}, \Thetal{l})},
\end{align*}
and the \method gradient is defined by 
\begin{align*}
\hn_{\Hl{l-1}}, \hn_{\Thetal{l}}
= \G{l}{\hn_{\Hl{l}}, \hat\Cv(\Hl{l-1}, \Thetal{l})},
\end{align*}
where $\hn_{\Hl{L}}=\nabla_{\Hl{L}}$. 
The batch loss is $\Lc_{\Dc}(\Thetav)$ and the batch gradient is $\nabla_{\Thetav} \Lc_{\Dc}(\Thetav)$. 
Assume that the FP gradient is unbiased, i.e., $\E{\nabla_{\Thetal{l}}}=\nabla_{\Thetal{l}} \Lc_{\Dc}(\Thetav)$ for all $l$. 

We first prove the following lemma
\begin{lemma}\label{lemma}
If $\E{\hn_{\Hl{l}}}=\E{\nabla_{\Hl{l}}}$, then there exists $\hCl{l}$, s.t.,
$\E{\hn_{\Hl{l-1}}}=\E{\nabla_{\Hl{l-1}}}$ and
$\E{\hn_{\Thetal{l}}}=\E{\nabla_{\Thetal{l}}}$.
\end{lemma}
\begin{proof}
By the chain rule of differentiation, we have 
\begin{align}
\nabla_{H^{(l-1)}_{ij}} = \sum_{kl} \tfrac{\partial H^{(l)}_{kl}}{\partial H^{(l-1)}_{ij}} \nabla_{H^{(l)}_{kl}},~~~~~
\nabla_{\Theta^{(l)}_{i}} = \sum_{kl} \tfrac{\partial H^{(l)}_{kl}}{\partial \Theta^{(l)}_{i}} \nabla_{H^{(l)}_{kl}}.
\end{align}
Therefore, we can write 
\begin{align*}
\G{l}{\nabla_{\Hl{l}}, \Cv(\Hl{l-1}, \Thetal{l})}
= \{\sum_{kl} \tfrac{\partial H^{(l)}_{kl}}{\partial H^{(l-1)}_{ij}} \nabla_{H^{(l)}_{kl}}\}_{ij}, 
\{\sum_{kl} \tfrac{\partial H^{(l)}_{kl}}{\partial \Theta^{(l)}_{i}} \nabla_{H^{(l)}_{kl}}\}_i,
\end{align*}
where $\Cv(\Hl{l-1}, \Thetal{l})=\{\sfrac{\partial H^{(l)}_{kl}}{\partial H^{(l-1)}_{ij}}, \sfrac{\partial H^{(l)}_{kl}}{\partial \Theta^{(l)}_{i}}\}_{ijkl}$.
Let $\hat\Cv(\Hl{l-1}, \Thetal{l})=Q(\Cv(\Hl{l-1}, \Thetal{l}))$, where
$Q(\cdot)$ is an unbiased quantizer, i.e., for all $\xv$, $\E{Q(\xv)}=\xv$. Then, we have
\begin{align*}
&\E{\hn_{\Hl{l-1}}, \hn_{\Thetal{l}}}=
\E{\G{l}{\hn_{\Hl{l}}, \hat\Cv(\Hl{l-1}, \Thetal{l})}}
=\E{\{\sum_{kl} Q(\tfrac{\partial H^{(l)}_{kl}}{\partial H^{(l-1)}_{ij}}) \hn_{H^{(l)}_{kl}}\}_{ij},  
	\{\sum_{kl} Q(\tfrac{\partial H^{(l)}_{kl}}{\partial \Theta^{(l)}_{i}}) \hn_{H^{(l)}_{kl}}\}_i}\\
=& \{\sum_{kl} \Eb Q(\tfrac{\partial H^{(l)}_{kl}}{\partial H^{(l-1)}_{ij}}) \Eb \hn_{H^{(l)}_{kl}}\}_{ij},  
\{\sum_{kl} \Eb Q(\tfrac{\partial H^{(l)}_{kl}}{\partial \Theta^{(l)}_{i}}) \Eb \hn_{H^{(l)}_{kl}}\}_i
= \{\sum_{kl} \tfrac{\partial H^{(l)}_{kl}}{\partial H^{(l-1)}_{ij}} \nabla_{H^{(l)}_{kl}}\}_{ij},  
\{\sum_{kl} \tfrac{\partial H^{(l)}_{kl}}{\partial \Theta^{(l)}_{i}} \nabla_{H^{(l)}_{kl}}\}_i\\
=& \G{l}{\nabla_{\Hl{l}}, \Cv(\Hl{l-1}, \Thetal{l})} = \nabla_{\Hl{l-1}}, \nabla_{\Thetal{l}}.
\end{align*}
\end{proof}

Now we can prove Theorem~\ref{thm:bias}.
\begin{proof}
By definition, $\hn_{\Hl{L}}=\nabla_{\Hl{L}}$, so $\E{\hn_{\Hl{L}}}=\E{\nabla_{\Hl{L}}}$. Using Lemma~\ref{lemma} and mathematical induction, we get 
$\E{\hn_{\Thetal{l}}}=\E{\nabla_{\Thetal{l}}}$, for all $l\in \{1, \dots, L\}$, so $\E{\hn_{\Thetav}}=\E{\nabla_{\Thetav}}$.

On the other hand, $\E{\nabla_{\Thetav}}=\nabla \Lc_{\Dc}(\Thetav)$, by the assumption. Put it all together, we have 
$\E{\hn_{\Thetav}} = \nabla \Lc_{\Dc}(\Thetav)$.
\end{proof}

\subsection{Theorem~\ref{thm:convergence}}
Let $\Thetav_t=\{\Thetal{l}\}_{l=1}^L$ be a flattened vector of the parameters at the $t$-th iteration, and $\hn_{\Thetav_t}=\{\hn_{\Thetav_t^{(l)}}\}_{l=1}^L$ be the corresponding AC gradient. For any vector $\xv$, let $\Var{\xv}:=\Eb\norm{\xv}^2 - \norm{\E{\xv}}^2$. Let $\Lc(\Thetav_t)$ be the batch loss at the $t$-th iteration, where $\nabla_{\Thetav} \Lc(\Thetav_t)=\E{\hn_{\Thetav_t}}$. Assume the SGD iteration $\Thetav_{t+1}\leftarrow \Thetav_{t}-\alpha\hn_{\Thetav_t}$. Further, let $\Econd{\cdot}{t}$ and $\Varcond{\cdot}{t}$ be the expectation and variance taken only over the minibatch and random quantizations at the $t$-th iteration. 
\begin{proof}
According to~\citet{bottou2018optimization}, A1 implies that for any $\Thetav_t$ and $\Thetav_{t+1}$,
\begin{align*}
\Lc(\Thetav_{t+1}) - \Lc(\Thetav_t) \le \nabla\Lc(\Thetav_t)^\top(\Thetav_{t+1}-\Thetav_t) + \frac{1}{2}\beta\norm{\Thetav_{t+1}-\Thetav_t}^2.
\end{align*}
Plugging the SGD iteration, we have
\begin{align*}
\Lc(\Thetav_{t+1}) - \Lc(\Thetav_t) \le  -\alpha\Lc(\Thetav_t)^\top \hn_{\Thetav_t} + \frac{1}{2}\alpha^2\beta\norm{\hn_{\Thetav_t}}^2,
\end{align*}
taking expectation w.r.t. iteration $t+1$ on both sides, and utilizing A3, the definition of variance, the step size inequality, and the unbiased AC gradient, 
\begin{align*}
\Econd{\Lc(\Thetav_{t+1})}{t+1} - \Lc(\Thetav_t) &\le  -\alpha\norm{\Lc(\Thetav_t)}^2 + \frac{1}{2}\alpha^2\beta (\Varcond{\hn_{\Thetav_t}}{t+1} + \norm{\Econd{\hn_{\Thetav_t}}{t+1}}^2)\\
&= -\alpha(1 - \frac{1}{2}\alpha\beta)\norm{\Lc(\Thetav_t)}^2+\frac{1}{2}\alpha^2\beta\sigma^2\\
&\le -\frac{1}{2}\alpha\norm{\Lc(\Thetav_t)}^2+\frac{1}{2}\alpha^2\beta\sigma^2.
\end{align*}
Taking expectation on both sides, we have
\begin{align}
\E{\Lc(\Thetav_{t+1})} - \E{\Lc(\Thetav_t)} \le -\frac{1}{2}\alpha\Eb\norm{\Lc(\Thetav_t)}^2+\frac{1}{2}\alpha^2\beta\sigma^2.
\label{eqn:thm-recursion}
\end{align}
Summing Eq.~(\ref{eqn:thm-recursion}) up across iterations $\{1, \dots, T\}$, and utilize A2, we have
\begin{align*}
L_{\inf} - \Lc(\Thetav_1)\le  \E{\Lc(\Thetav_{T+1})} - \E{\Lc(\Thetav_1)} \le -\frac{1}{2}\alpha\sum_{t=1}^T \Eb\norm{\Lc(\Thetav_t)}^2+\frac{1}{2}T\alpha^2\beta\sigma^2.
\end{align*}
Rearrange the terms, we have 
\begin{align*}
\frac{1}{T}\sum_{t=1}^T \Eb\norm{\Lc(\Thetav_t)}^2 \le \frac{2(\Lc(\Theta_1) - \Lc_{inf})}{\alpha T} + \alpha\beta\sigma^2.
\end{align*}
Viewing $t$ as a random variable, we have Eq.~(\ref{eqn:thm-1}).

\end{proof}

\subsection{Theorem~\ref{thm:grad-var}}
Let $\Econd{X}{Y}$ and $\Varcond{X}{Y}$ be the conditional expectation / variance of $X$ given $Y$. We use the following proposition.
\begin{proposition}
(Law of Total Variance)
$$\Var{X}=\E{\Varcond{X}{Y}} + \Var{\Econd{X}{Y}}.$$
\end{proposition}

Define 
\begin{align*}
\GT{l\sim m}{\hn_{\Hl{m}}, \hat\Cv^{(m)}}
	&= \GT{l}{\GH{l+1}{\cdots
			\GH{m}{\hn_{\Hl{m}}, \hCl{m} }
			\cdots , \Cl{l+1}}, \Cl{l}}, \\
\GT{l\sim m}{\hn_{\Hl{m}}}
		&= \GT{l}{\GH{l+1}{\cdots
				\GH{m}{\hn_{\Hl{m}}, \Cl{m} }
			\cdots	, \Cl{l+1}}, \Cl{l}},
\end{align*}

\begin{proof}
(of Theorem~\ref{thm:grad-var})
First, we have 
\begin{align}\label{eqn:thm3-boundary}
\Var{\GT{l\sim L}{\hn_{\Hl{L}}}} = \Var{\GT{l\sim L}{\nabla_{\Hl{L}}}} = \Var{\nabla_{\Hl{l}}}.
\end{align}
For all $m<L$, by definition of $\hn_{\Hl{m}}$
\begin{align*}
\Var{\GT{l\sim m}{\hn_{\Hl{m}}}} = 
\Var{\GT{l\sim m}{ \GH{m+1}{ \hn_{\Hl{m+1}}, \hCl{m+1} } }},
\end{align*}
by law of total variance
\begin{align*}
&\Var{\GT{l\sim m}{ \GH{m+1}{ \hn_{\Hl{m+1}}, \hCl{m+1} } }}\\
=& \E{\Varcond{\GT{l\sim m}{ \GH{m+1}{ \hn_{\Hl{m+1}}, \hCl{m+1} } }}{\hn_{\Hl{m+1}}}}
+ \Var{\Econd{\GT{l\sim m}{ \GH{m+1}{ \hn_{\Hl{m+1}}, \hCl{m+1} } }}{\hn_{\Hl{m+1}}}}
\end{align*}
by definition of $\GT{l\sim m}{\hn_{\Hl{m}}, \hat\Cv^{(m)}}$, definition of $\GT{l\sim m}{\hn_{\Hl{m}}}$ and Theorem~\ref{thm:bias}
\begin{align}
= \E{\Varcond{\GT{l\sim m+1}{ \hn_{\Hl{m+1}}, \hCl{m+1}  }}{\hn_{\Hl{m+1}}}}
+ \Var{\GT{l\sim m+1}{ \hn_{\Hl{m+1}} } }.\label{eqn:thm3-recursion}
\end{align}
Combining Eq.~(\ref{eqn:thm3-boundary}) and Eq.~(\ref{eqn:thm3-recursion}), we have
\begin{align}\label{eqn:thm3-general}
\Var{\GT{l\sim m}{\hn_{\Hl{m}}}} = \Var{\nabla_{\Hl{l}}} + \sum_{j=m+1}^{L}\E{\Varcond{\GT{l\sim j}{ \hn_{\Hl{j}}, \hCl{j}  }}{\hn_{\Hl{j}}}}.
\end{align}

Similarly, by definition and the law of total variance,
\begin{align*}
&\Var{\hn_{\Thetal{l}}}=\Var{\GT{l}{\hn_{\Hl{l}},\hCl{l}}}
= \E{\Varcond{\GT{l}{\hn_{\Hl{l}},\hCl{l}}}{\hn_{\Hl{l}}}}
+ \Var{\Econd{\GT{l}{\hn_{\Hl{l}},\hCl{l}}}{\hn_{\Hl{l}}} }.
\end{align*}.
by definition of $\GT{l\sim m}{\hn_{\Hl{m}}, \hat\Cv^{(m)}}$, definition of $\GT{l\sim m}{\hn_{\Hl{m}}}$ and Theorem~\ref{thm:bias}
\begin{align}\label{eqn:thm3-onestep}
=& \E{\Varcond{\GT{l}{\hn_{\Hl{l}},\hCl{l}}}{\hn_{\Hl{l}}}}
+ \Var{\GT{l}{\hn_{\Hl{l}},\Cl{l}}} \\
=& \E{\Varcond{\GT{l\sim l}{\hn_{\Hl{l}},\hCl{l}}}{\hn_{\Hl{l}}}}
+ \Var{\GT{l\sim l}{\hn_{\Hl{l}}}}.
\end{align}
Plugging Eq.~(\ref{eqn:thm3-general}) into Eq.~(\ref{eqn:thm3-onestep}), we have 

\begin{align*}
&\Var{\hn_{\Thetal{l}}} = \E{\Varcond{\GT{l\sim l}{\hn_{\Hl{l}},\hCl{l}}}{\hn_{\Hl{l}}}} + 
\Var{\nabla_{\Hl{l}}} + \sum_{j=l+1}^{L}\E{\Varcond{\GT{l\sim j}{ \hn_{\Hl{j}}, \hCl{j}  }}{\hn_{\Hl{j}}}}\\
=& \Var{\nabla_{\Hl{l}}} + \sum_{m=l}^{L}\E{\Varcond{\GT{l\sim m}{ \hn_{\Hl{m}}, \hCl{m}  }}{\hn_{\Hl{m}}}}.
\end{align*}
\end{proof}

\section{Bias and Variance for Individual Layers}\label{sec:layers}

\subsection{Convolutional Layers}
Consider an arbitrary dimensional  convolutional layer
\begin{align}
\yv_{nia} = \sum_{\Delta_i}  \Wv_{\Delta_i,a}\xv_{n,si+d\Delta_i,a}.
\label{eqn:conv-definition}
\end{align}
where $\Xv=\{\xv_{nia}\}$ is the input, $\Yv=\{\yv_{nia}\}$ is the output, $s$ is a stride, and $d$ is a dilation, $a\in [0, A)$ is the group index for depthwise-separable convolution. 
$\xv_{ni}$ is the feature vector of the $n$-th sample at the $i$-th location, where $i$ can be a tuple of arbitrary dimensions, e.g., 2 or 3.
The kernel has $K=|\Delta_i|$ locations, and for each location, the kernel $\Wv_{\Delta_i,a}$ is a $(D_{out}/A)\times (D_{in}/A)$ matrix.

The gradients are
\begin{align}
\nabla_{\Wv_{\Delta_i},a} = 
\sum_{ni} \nabla_{\yv_{nia}} \xv_{n,si+d\Delta_i,a}^\top,~~~~\nabla_{\xv_{nia}}=\sum_{\Delta_i}\nabla_{\yv_{n,(i-d\Delta_i)/s,a}} \Wv_{\Delta_i,a}^\top.
\label{eqn:conv-gradient}
\end{align}
Define the approximate context as $\hat\Cv=(Q(\Xv), \Wv)$, where $Q(\cdot)$ is an unbiased quantizer. Then,
\begin{align*}
\E{\hn_{\Wv_{\Delta_i,a}}}=\sum_{ni} \E{\hn_{\yv_{nia}}}\E{Q(\xv_{n,si+d\Delta_i,a}^\top)}
= \sum_{ni}\nabla_{\yv_{nia}}\xv_{n,si+d\Delta_i,a}^\top.=\E{\nabla_{\Wv_{\Delta_i, a}}}.
\end{align*}
Therefore, the gradient is unbiased.

Let $I$ be the number of locations (pixels) on the feature map $\Xv$, and $S$ is the product of strides. The variance is 
\begin{align*}
\Var{\sum_{ni} \nabla_{\yv_{nia}} \xv_{n,si+d\Delta_{i},a}^\top} = \sum_{c_1c_2}\Var{\sum_{ni} \nabla_{y_{n,i,a,c_1}} x_{n,si+d\Delta_{i},a,c_2}}
\end{align*}
Due to independence,
\begin{align*}
\Var{\sum_{ni} \nabla_{\yv_{nia}} \xv_{n,si+d\Delta_{i},a}^\top} =& \sum_{c_1c_2ni} \nabla_{y_{n,i,a,c_1}}^2 \Var{x_{n,si+d\Delta_{i},a,c_2}}\\
=& \frac{G}{6} \sum_{ni} \norm{\nabla_{\yv_{n,i,a}}}^2  \norm{\Rv_{n,si+d\Delta_i,a}}^2 \approx \frac{G}{6I} \sum_{n} \norm{\nabla_{\yv_{na}}}^2  \norm{\Rv_{na}}^2,
\end{align*}
where we approximate $\sum_i a_ib_i\approx \mathrm{sum}(a_i)\mathrm{mean}(b_i)$.
Therefore, 
\begin{align*}
\Var{\nabla_{\Wv}} = \sum_{\Delta_i,a} \nabla_{\Wv_{\Delta_i,a}} \approx 
\frac{GK}{6I} \sum_{na} \norm{\nabla_{\yv_{na}}}^2  \norm{\Rv_{na}}^2
\approx \frac{GK}{6IA}\sum_n\norm{\nabla_{\yv_{n}}}^2  \norm{\Rv_{n}}^2.
\end{align*}

\noindent\textbf{Transposed Convolution} We can view transposed convolution as convolutions with inverse stride. For example, if a Conv2D has the stride $[2, 2]$, then its transpose has the stride $[1/2, 1/2]$.

\subsection{Normalization Layers}

Suppose the input is $\Xv\in \Rb^{N\times C}$, where there are $C$ features to be normalized. The layer has the weight $\wv\in \Rb^{C}$ and the bias $\bv\in \Rb^{C}$. The output is $\Yv\in \Rb^{N\times C}$. For example, for \texttt{BatchNorm2d}, $C$ is the number of channels, and $N$ is the product of the batch size and the number of pixels per image. This formulation applies to both batch normalization and layer normalization, for arbitary-shape tensors. 

\noindent\textbf{Forward Propagation}
\begin{align}\label{eqn:bn-forward}
y_{nc} = (x_{nc}-m_c)\frac{w_c}{s_c} + b_c, \mbox{ where } m_c=\frac{1}{N}\sum_{n}x_{nc},  s_c=\sqrt{\frac{1}{N}\sum_n\left(x_{nc}-m_c\right)^2}.
\end{align}

\noindent\textbf{Back Propagation}
$$
\nabla_{x_{nc}}
= \frac{w_c}{s_c}
\left(
\nabla_{y_{nc}}-\frac{1}{N}\sum_{c^\prime} 
\nabla_{y_{nc^\prime}}
- \frac{1}{N s_c^2} (x_{nc}-m_c)
\sum_{n^\prime} (x_{n^\prime c}-m_{c})\nabla_{y_{n^\prime c}}
\right)
$$

The context is $(\Xv, \mv, \sv, \wv)$. We only quantize $\Xv$ here since all the other vectors are negligible in size.

\noindent\textbf{Unbiased Quantization}
We can keep two independently quantized copies of $x_{nc}$: $\hat x_{nc}$ and $\dot x_{nc}$, such that $\E{\hat x_{nc}}=x_{nc}$, $\E{\dot x_{nc}}=x_{nc}$. In this way,
\begin{align*}
\E{\hn_{x_{nc}}}
= \E{\frac{w_c}{s_c}
\left(
\hn_{y_{nc}}-\frac{1}{N}\sum_{c^\prime} 
\hn_{y_{nc^\prime}}
- \frac{1}{N s_c^2} (\hat x_{nc}-m_c)
\sum_{n^\prime} (\dot x_{n^\prime c}-m_{c})\hn_{y_{n^\prime c}}
\right)}
\end{align*}
Due to independence, 
\begin{align*}
\E{\hn_{x_{nc}}}
&= \frac{w_c}{s_c}
	\left(
	\E{\hn_{y_{nc}}}-\frac{1}{N}\sum_{c^\prime} 
	\E{\hn_{y_{nc^\prime}}}
	- \frac{1}{N s_c^2} \E{\hat x_{nc}-m_c}
	\sum_{n^\prime} \E{\dot x_{n^\prime c}-m_{c}}\E{\hn_{y_{n^\prime c}}}
	\right)\\
&= \frac{w_c}{s_c}
\left(
\nabla_{y_{nc}}-\frac{1}{N}\sum_{c^\prime} 
\nabla_{y_{nc^\prime}}
- \frac{1}{N s_c^2} (x_{nc}-m_c)
\sum_{n^\prime} (x_{n^\prime c}-m_{c})\nabla_{y_{n^\prime c}}
\right) = \nabla_{x_{nc}}.
\end{align*}
Therefore, the gradient of the input is unbiased.

\textbf{Gradient Variance}
\begin{align*}
\Var{\hn_{x_{nc}}}
&=\frac{w_c^2}{N^2 s_c^6}
\Var{
	(\hat x_{nc}-m_c)
	\sum_{n^\prime} (\dot x_{n^\prime c}-m_{c})\nabla_{y_{n^\prime c}}
} \\
&= \frac{w_c^2}{N^2 s_c^6}\left(\Var{A}\Var{B} + \E{A}^2 \Var{B} + \Var{A} \E{B}^2\right),
\end{align*}
utilizing $\Var{AB}=\Var{A}\Var{B}+\E{A}^2 \Var{B} + \Var{A} \E{B}^2$, if $A$ and $B$ are independent, 
where 
\begin{align*}
A &= \hat x_{nc}-m_c,~~~~
\E{A} = x_{nc} - m_c,~~~~
\Var{A} = \Var{\hat x_{nc}}\\
B &=\sum_{n^\prime} (\dot x_{n^\prime c}-m_{c})\nabla_{y_{n^\prime c}},~~~~
\E{B}=\sum_{n^\prime} ( x_{n^\prime c}-m_{c})\nabla_{y_{n^\prime c}},~~~~
\Var{B}=\sum_{n^\prime}\Var{\dot x_{n^\prime c}}\nabla_{y_{n^\prime c}}^2.
\end{align*}
Assume that $\Var{A}\ll \E{A}^2$ and $\Var{B}\ll \E{B}^2$, and utilize Eq.~(\ref{eqn:bn-forward}), we have
\begin{align*}
\Var{\hn_{X_{nc}}}
\approx \frac{w_c^2}{N^2 s_c^4}
\left[
y_{nc}^2 \sum_{n^\prime} \Var{\dot x_{n^\prime c}}\nabla_{y_{n^\prime c}}^2 +
\Var{\hat x_{nc}}\left(
\sum_{n^\prime} y_{n^\prime c}\nabla_{y_{n^\prime c}}
\right)^2
\right].
\end{align*}
Let $d_c=\sum_{n^\prime} y_{n^\prime c}\nabla_{y_{n^\prime c}}$, and plug $\Var{\hat x_{nc}}\approx \frac{R_n^2}{6B_n^2}$ in, we have 
\begin{align*}
\Var{\nabla_{X_{nc}}}
\approx \frac{w_c^2}{6N^2 s_c^4}
\left[
y_{nc}^2 \sum_{n^\prime} \frac{R_{n^\prime}^2}{B_{n^\prime}^2}\nabla_{y_{n^\prime c}}^2 +
\frac{R_{n}^2}{B_{n}^2}d_c^2
\right].
\end{align*}
Summing the terms up, we have
\begin{align*}
\Var{\nabla_{\Xv}}=\sum_{nc}\Var{\nabla_{x_{nc}}}=
\frac{1}{6N^2}\sum_{n^\prime} \frac{R_{n^\prime}^2}{B_{n^\prime}^2}\left(
\sum_c \frac{w_c^2}{s_c^4} \nabla_{y_{n^\prime c}}^2 \sum_n y_{nc}^2 
\right)+\sum_{n} \frac{R_n^2}{B_n^2}\sum_c \frac{w_c^2}{s_c^4}d_c^2.
\end{align*}
Finally, noticing that $\sum_n Y_{nc}^2 = Nw_c^2$, and rearrange the terms, we have
\begin{align*}
\Var{\nabla_{\Xv}}
=\frac{1}{6N^2}\sum_{n} \frac{R_{n}^2}{B_{n}^2}\left(
\sum_c \frac{w_c^2}{s_c^4}
\left(
N w_c^2 \nabla_{y_{nc}}^2 + d_c^2
\right)
\right).
\end{align*}
This can be computed by keeping track of $ \sum_c w_c^4/s_c^4\nabla_{y_{nc}}^2$ for each $n$ and $d_c^2$ for each $d$.
In practice, we may not be able to record the gradient for every sample. In this case, we approximate the gradient-related terms with a constant,
\begin{align*}
\Var{\nabla_{\Xv}}
\propto\sum_{n} \frac{R_{n}^2}{B_{n}^2}.
\end{align*}

\noindent\textbf{Biased Quantization}
As maintaining two quantized copies is too expensive, we only maintain one copy in practice. The gradient is still close to unbiased in this case. To see this,
\begin{align*}
\E{\hn_{x_{nc}}}
&= \frac{w_c}{s_c}
	\left(
	\hn_{y_{nc}}-\frac{1}{N}\sum_{c^\prime} 
	\hn_{y_{nc^\prime}}
	- \frac{1}{N s_c^2}\E{ (\hat x_{nc}-m_c)
	\sum_{n^\prime} (\hat x_{n^\prime c}-m_{c})\hn_{y_{n^\prime c}}}
	\right)\\
&= \frac{w_c}{s_c}
\left(
\hn_{y_{nc}}-\frac{1}{N}\sum_{c^\prime} 
\hn_{y_{nc^\prime}}
- \frac{1}{N s_c^2}\left( (x_{nc}-m_c)
	\sum_{n^\prime} (x_{n^\prime c}-m_{c})\hn_{y_{n^\prime c}}
	\right)+ \Var{\hat x_{nc}}\hn_{y_{nc}}\right)
\end{align*}
As $N$ is huge, the additional term $\Var{\hat x_{nc}}\hn_{y_{nc}}$ is negligible comparing to other terms. 

\subsection{Activation Layers}
Activation layers take the form 
\begin{align*}
y_i = f(x_i),~~~~\nabla_{x_i} = f'(x_i) \nabla_{y_i},
\end{align*}
where $f(\cdot)$ is an activation function. We can simply store a quantized version of $\{f'(x_i)\}$ as the context for unbiased gradient.

ReLU layers are particularly simple, which have $f'(x_i) = \Ib(x_i>0)$. Therefore, ReLU layers only take a single bit per dimension to store, without any approximation. 

\subsection{Pooling Layers}
Pooling layers are computed without any approximation. We don't need to quantize their context because their context only takes little memory.
We discuss how many bits are required for their context below.

\noindent\textbf{Average pooling layers}
An average pooling layer can be seen as a special case of convolution layer with a constant kernel.
Because the kernel is a constant, we do not need to compute the gradient of kernel.
To compute the gradient of input, according to Eq. \ref{eqn:conv-gradient}, we only need the gradient of output. Therefore, we can compute average pooling pooling layers exactly without saving anything in the context.

\noindent\textbf{Max Pooling}
Following the notation in Eq. \ref{eqn:conv-definition}. A max pooling layer takes the form

\begin{align*}
y_{nij} = \max_{\Delta_i}{x_{n,si+d\Delta_i,j}}
\end{align*}

The gradient of input is
\begin{align*}
\nabla_{x_{nij}} = \sum_{\Delta_i}{ \nabla_{y_{n,(i-d\Delta_i)/s,j}} \Ib( \argmax_k{x_{n,i-d\Delta_i + dk,j}} = \Delta_i) }
\end{align*}

For each output location $y_{nij}$, we need to store an integer value $k_{nij} = \argmax_{\Delta_i}{x_{n,si+d\Delta_i,j}}$. Note that $0 \le k_{nij} < K$, where K is the kernel size of pooling. To store $k_{nij}$, we need $\ceil{log(K)}$ bits per output location (pixel).
In common neural networks, the kernel size of a max pooling layer is less than $8 \times 8=256$. We use 8 bits per output location in our implementation.

\section{Experimental Setup and Additional Experiments}\label{sec:more-experiments}
\subsection{Experimental Setup}\label{sec:exp-setup}
We use standard open-source model architecture and training recipes for all the tasks.

\noindent\textbf{Quantization Strategy}
We use the ResNet50-v1.5 repo\footnote{\url{https://github.com/NVIDIA/DeepLearningExamples/tree/master/PyTorch/Classification/ConvNets/resnet50v1.5}} for ImageNet experiments. The batch size is $32\times 8=256$, and the initial learning rate is $0.256$. We train 90 epochs with 4 warmup epochs. The repo further has cosine learning rate schedule and label smoothing. 

\noindent\textbf{Computational Overhead} We use ResNet-50, ResNet-152, WideResNet-101 and DenseNet-201 from the \textbf{torchvision} package.
We convert them to use ActNN's layers by our model convertor.
We run five training batches on ImageNet and report the median of the training throughput (images per second).

\noindent\textbf{Semantic Segmentation and Object Detection}
We use the open-source frameworks MMSegmentation~\cite{mmseg2020} and MMDetection~\cite{mmdetection} for these tasks, and follow the original training recipes. For semantic segmentation, the crop size is $512\times1024$. The methods are FCN~\cite{shelhamer2017fully} and FPN~\cite{kirillov2019panoptic}.
Backbones are chosen from HRNetV2W48 (HRNet)~\cite{wang2020deep}, ResNet-50 dilation 8 (Dilation8), and original ResNet-50 (FPN). For object detection, the input size is 800 pixels for the short edge. The model is RetinaNet~\cite{lin2017focal} with FPN as the backbone.

\subsection{Variance Profile}\label{sec:var-profile}
We visualize all the terms in Thm.~\ref{thm:grad-var} in Fig.~\ref{fig:var-profile} for a ResNet-50 trained on ImageNet at the 50-th epoch. The quantization strategy is 2-bit per-group quantization (\method L2 in Tab.~\ref{tab:opt_level}).
 In the figure, each row is a stochasticity and each column is a parameter. For example, the entry at the $m$-th row and the $l$-th column is the term $\E{
	\Varcond{ \GT{l\sim m}{
			\hn_{\Hl{m}}, \hat\Cv^{(m)}
	}}{\hn_{\Hl{m}}}
}$, i.e., the impact of the $m$-th layer's quantized activation to the gradient variance of the $l$-th layer's parameter. 
The last row is the minibatch sampling variance $\Var{ \nabla_{\Thetal{l}}}$. From the figure we can observe that
\begin{enumerate}
	\item Minibatch sampling variance is much higher than the quantization variance. Therefore, it is possible to train with compressed activations, without impacting the final accuracy;
	\item The quantization variance for each layer is dominantly impacted by compressing the activation at the same layer. Therefore, our strategy in Sec.~\ref{sec:fine-grined}, which omits all the distant terms, approximates the exact variance well.
\end{enumerate}

\begin{figure*}[p]
\centering
\includegraphics[height=20cm]{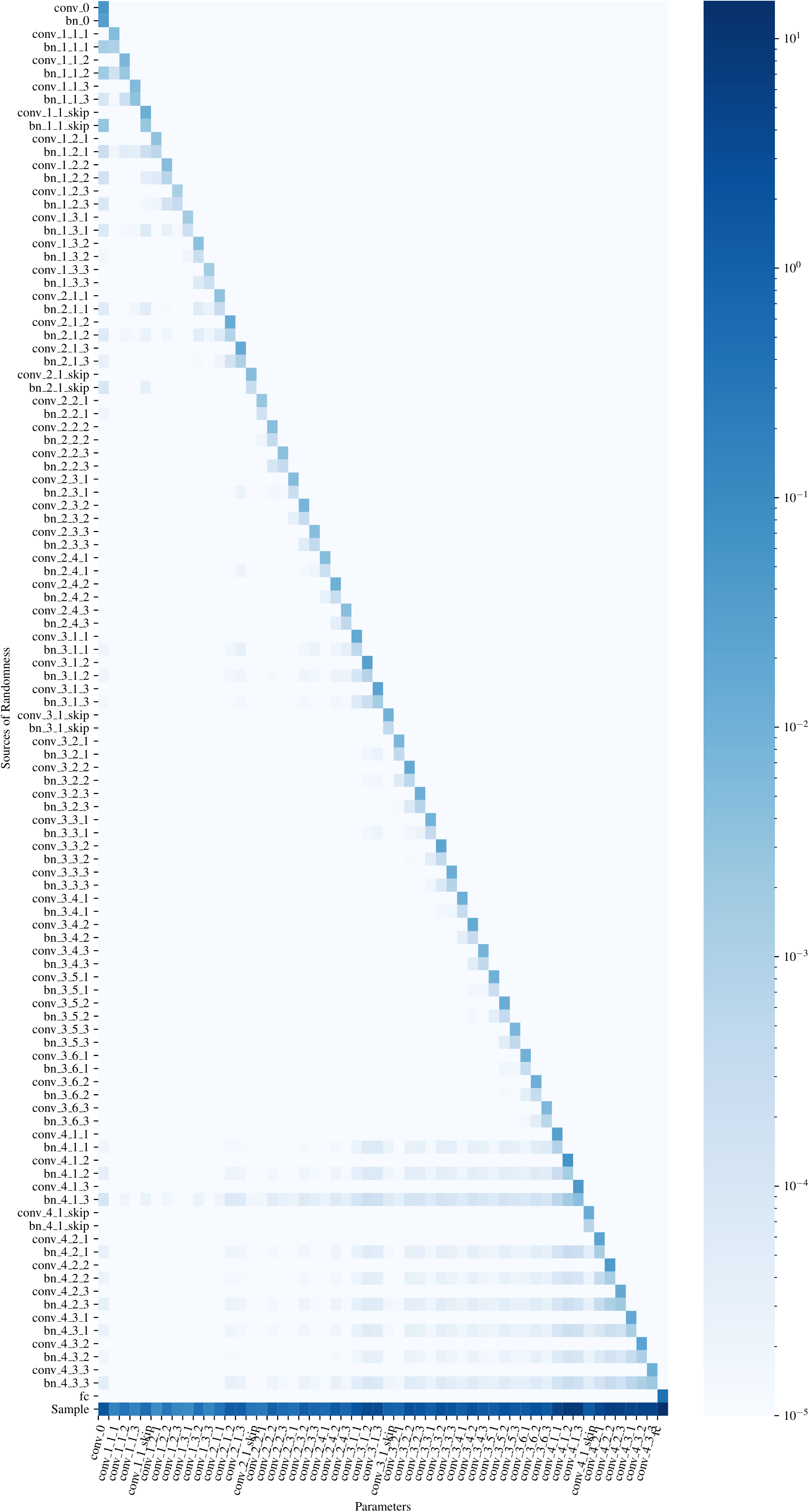}\label{fig:var-profile}
\caption{A decomposition of variance according to Thm.~\ref{thm:grad-var}. Each row is a source of randomness, either quantization or minibatch sampling (last row). Each column is a parameter gradient, which we would like measure variance of. Each entry is the impact of one source of randomness to one layer's parameter gradient.}\label{fig:variance-profile}
\end{figure*}

\subsection{CIFAR-10 Results}\label{sec:cifar10}
In Fig.~\ref{fig:more-gradvar}, we additional present results on CIFAR-10. The conclusion remains the same with the CIFAR-100 and ImageNet experiments in the main text. BLPA converges only at 4 bits, while \method converges at 2 bits.

\begin{figure*}[t]
	\centering
	\begingroup
	\setlength{\tabcolsep}{0pt} 
	\scriptsize
	\begin{tabular}{m{5cm}m{5cm}m{5cm}m{2.4cm}}
		\includegraphics[width=.95\linewidth]{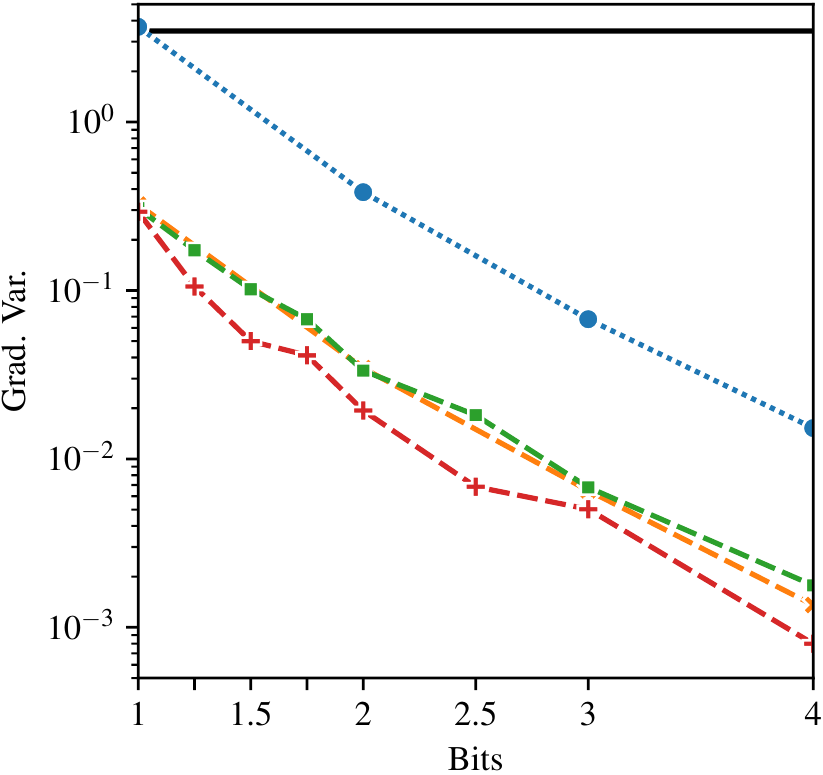}&
		\includegraphics[width=.95\linewidth]{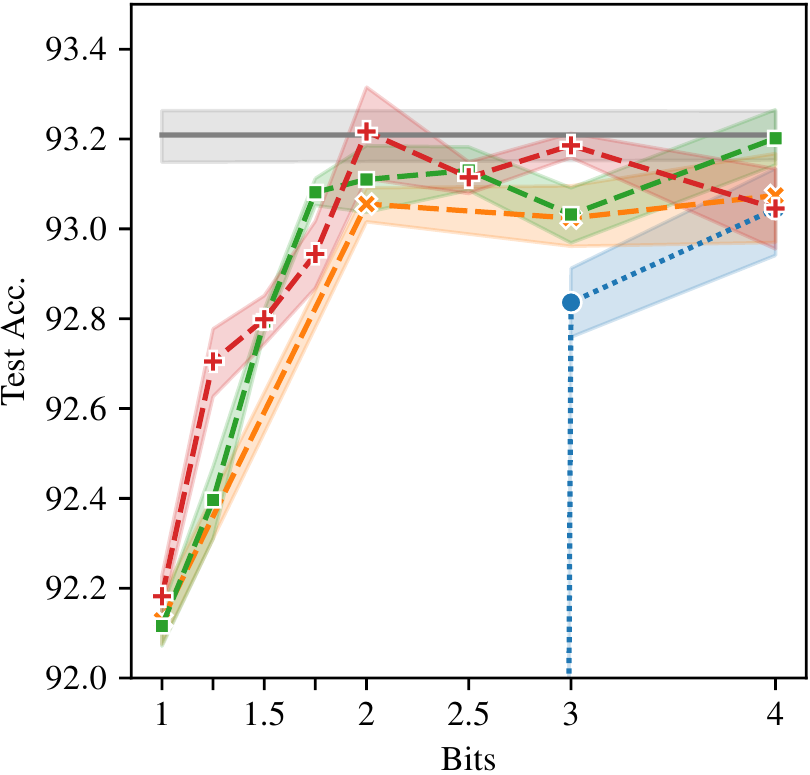}&
		\includegraphics[width=.95\linewidth]{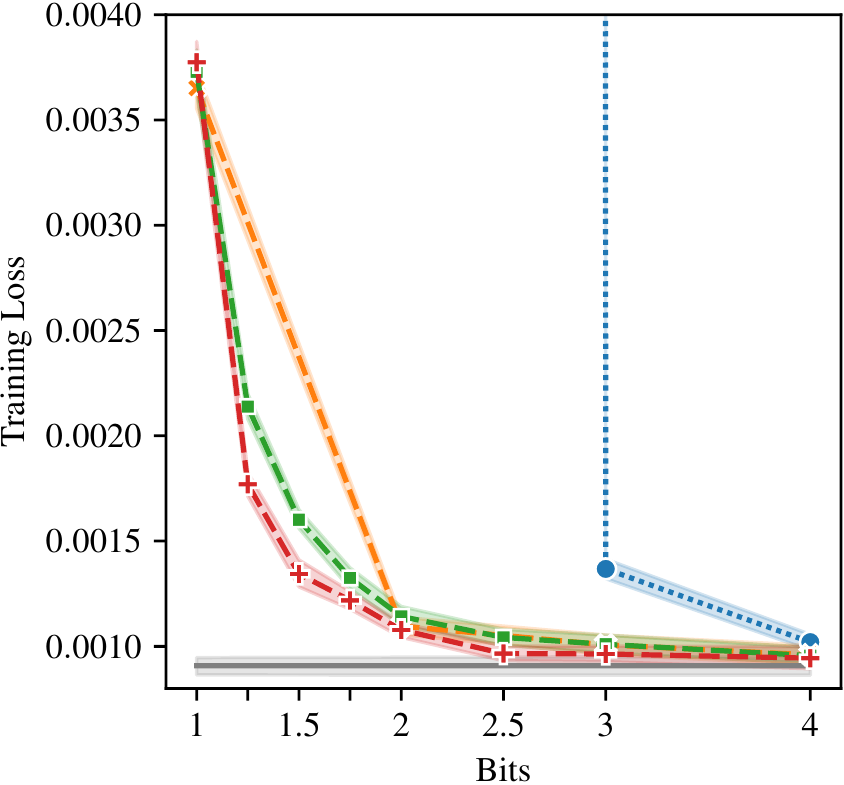}&
		\includegraphics[width=\linewidth]{figures/legend.pdf}\\
		(a) Gradient variance on CIFAR-10 & (b) Testing accuracy on CIFAR-10 & (c) Testing loss on CIFAR-10 & 
	\end{tabular}
	\endgroup
	\caption{Ablation study on the quantization strategy on CIFAR-10. BLPA diverges with 1 and 2 bits. The gradient variance is calculated at the 10th epoch.  }
	\label{fig:more-gradvar}
\end{figure*}

\end{document}